\documentclass[twoside,11pt]{article}

\usepackage{jmlr2e}
\usepackage{subcaption}
\usepackage{url}
\usepackage{epsfig,epstopdf,algorithmic,algorithm,amsfonts,url,textcomp,amsmath,multirow,moresize,wrapfig}
\usepackage{graphicx}
\usepackage[tableposition=top]{caption}
\usepackage{subcaption}
\usepackage{enumitem,picinpar}
\usepackage[framemethod=tikz]{mdframed}
\usepackage[flushleft]{threeparttable}
\DeclareMathOperator*{\argmin}{argmin}

\newenvironment{myitemize}
{ \begin{center} \begin{em} }
{ \end{em} \end{center} }

\begin{document}

\title{Federated Doubly Stochastic Kernel Learning \\ for  Vertically Partitioned Data}

\author{\name Bin Gu \email jsgubin@gmail.com      \\
\addr JD Finance America Corporation \\
\name Zhiyuan Dang \email {zhiyuandang@gmail.com} \\
\addr School of Electronic Engineering, Xidian University, China \\
   \name Xiang Li \email {lxiang2@uwo.ca} \\
 \addr Computer Science Department, Western University, Ontario, Canada \\
   \name Heng Huang \email {heng@uta.edu} \\
\addr JD Finance America Corporation \\ University of Pittsburgh, USA\\
}

\editor{}
\maketitle

\begin{abstract}
In a lot of real-world data mining and machine learning applications, data are  provided by multiple providers and  each maintains private records of different feature sets about common entities. %, which is also called as vertically partitioned  data.
It is challenging  to train these vertically partitioned data effectively and efficiently while keeping  data privacy for traditional data mining and  machine learning algorithms. In this paper,  we focus on nonlinear learning with kernels, and    propose a  federated doubly stochastic kernel learning (FDSKL) algorithm for vertically partitioned data.
Specifically, we use random features to approximate the kernel mapping function  and use   doubly stochastic gradients to update the solutions, which  are all computed federatedly without the disclosure of data. %Note that the nonlinear prediction model is stored separately  in different workers.
Importantly, we prove that FDSKL has a sublinear convergence rate, and can guarantee the
data security under the semi-honest assumption. Extensive experimental results on  a variety of benchmark datasets  show  that FDSKL is significantly faster than state-of-the-art federated learning methods when dealing with kernels,  while  retaining the similar generalization performance.
\end{abstract}

\section{Introduction}\label{section_intro}
Vertically partitioned  data \citep{Skillicorn2008Distributed} widely exist in the modern data mining and machine learning applications, where data are provided by multiple (two or more) providers (companies, stakeholders, government departments, \emph{etc.}) and  each maintains the records of different feature sets with common entities. For example, a digital finance company, an E-commerce company,  and a bank collect different information about the same person. The digital finance company  has access
to the online consumption, loan and repayment   information. The E-commerce company has access
to the online shopping information. The bank has customer information such as average monthly deposit, account balance. If the person submit a loan application  to the digital finance company,
the digital finance company might evaluate the  credit risk  to this financial loan   by   comprehensively utilizing  the information stored in the three parts.

However,  direct access to the data in other providers or  sharing of the data may  be prohibited due to legal and commercial reasons. For legal reason, most countries  worldwide have made laws in protection of data security
and privacy. For example, the
European Union made the General Data Protection Regulation (GDPR) \citep{Regulation2018} to protect users' personal privacy and data
security. The recent data breach by Facebook has caused a wide range of protests \citep{badshah2018facebook}.
For commercial reason, customer data is usually a valuable
business asset for corporations. For example, the real online shopping information of customers can be used to train a  recommended model which could provide valuable product recommendations to customers.
Thus, all of these causes require  federated learning on vertically  partitioned data without the disclosure of data.

\begin{table*}[htbp!]
	%\vspace*{-12pt}
	\small
	\center \caption{Commonly used smooth loss functions for binary classification (BC)  and regression (R).  }
%	%\vspace*{-6pt}
	\setlength{\tabcolsep}{1mm}
	\begin{tabular}{|c|c|c|p{7cm}|}
		\hline
	 \textbf{Name of loss} & \textbf{Type of task}   &  \textbf{Convex}  &
		\textbf{The loss function}
		\\
		\hline \hline
	Square loss  & BC+R  & 	Yes  &  $L(f(x_i),y_i)= (f(x_i) - y_i)^2$  \\ \hline
		 Logistic  loss & BC & Yes& $
		L(f(x_i),y_i) = \log (1+ \exp (-y_i f(x_i)))
		$  \\ \hline
		Smooth hinge loss & BC & Yes & $L(f(x_i),y_i) =  \left \{ \begin{array} {l@{\
			}l} \frac{1}{2} - z_i   &
		\textrm{if } \ z_i  \leq 0
		\\ \frac{1}{2} (1 - z_i  )^2 &  \textrm{if } \  0< z_i  < 1
		\\ 0 &  \textrm{if } \   z_i  \geq 1  \end{array}
		\right.$, where $z_i = y_if(x_i)$.  \\
		\hline
	\end{tabular}
	\label{table:IP}
	%\vspace*{-12pt}
\end{table*}

In the literature, there are many privacy-preserving federated learning algorithms for vertically partitioned data in various applications, for example, cooperative statistical analysis \cite{du2001privacy}, linear regression \citep{gascon2016secure,karr2009privacy,sanil2004privacy,gascon2017privacy}, association rule-mining \citep{vaidya2002privacy}, k-means clustering \citep{vaidya2003privacy}, logistic regression \citep{hardy2017private,nock2018entity}, random forest \cite{liu2019federated}, XGBoost \cite{cheng2019secureboost} and support vector machine \cite{yu2006privacy}. From the optimization standpoint, Wan et al.  \citeyearpar{wan2007privacy} proposed privacy-preservation gradient descent algorithm for vertically partitioned data. Zhang et al. \citeyearpar{zhang2018feature} proposed a feature-distributed SVRG algorithm for high-dimensional linear classification.

% For example, Du and Atallah \citeyearpar{du2001privacy}
% proposed privacy-preserving cooperative statistical analysis for  vertically  partitioned data. Du et al.  \citeyearpar{du2004privacy} proposed privacy-preserving multivariate statistical analysis for
% linear regression and classification on  vertically  partitioned data.  Several references \citep{gascon2016secure,karr2009privacy,sanil2004privacy,gascon2017privacy} proposed privacy-preserving linear regression for vertically  partitioned data. Vaidya and Clifton  \citeyearpar{vaidya2002privacy} proposed privacy-preserving association
% rule mining for vertically  partitioned data. Vaidya and Clifton  \citeyearpar{vaidya2003privacy} proposed privacy-preserving k-means clustering clustering for vertically  partitioned data.  \cite{hardy2017private,nock2018entity} studied the effect of entity resolution and proposed a vertical federated learning scheme to train a privacy-preserving
% logistic regression model.  Wan et al.  \citeyearpar{wan2007privacy} proposed privacy-preservation gradient descent algorithm for vertically partitioned data. Zhang et al. \citeyearpar{zhang2018feature} proposed a feature-distributed SVRG algorithm for
% high-dimensional linear classification.

However, existing privacy-preservation federated learning algorithms assume the models are implicitly linearly separable, \textit{i.e.}, in  the form of $f(x)=g\circ h(x)$ \citep{wan2007privacy}, where $g$ is any differentiable function, $\{{\mathcal{G}_1},{\mathcal{G}_2},\ldots,{\mathcal{G}_q} \}$ is a partition of the features and $ h(x)$  is a linearly separable function of the form of $\sum_{\ell=1}^{q} h^\ell(w_{\mathcal{G}_\ell},x_{\mathcal{G}_\ell})$. Thus, almost all the existing privacy-preservation federated learning algorithms are  limited to   linear models. However, we know that the nonlinear models often can achieve better risk prediction performance than linear methods.
Kernel methods \cite{gu2018accelerated,gu2008line,gu2018new} are an important class of  nonlinear approaches to handle linearly non-separable data. Kernel models usually take the form of $f(x)=\sum_{i}^N \alpha_i K(x_i,x)$ which do not satisfy the assumption of implicitly linear separability, where $K(\cdot,\cdot)$ is a kernel function.   To the best of our knowledge, PP-SVMV \citep{yu2006privacy} is the only privacy-preserving method for learning vertically partitioned data with non-linear kernels. However, PP-SVMV \citep{yu2006privacy} needs to merge the local kernel matrices in  different workers to a global kernel matrix, which would cause a high communication cost.  It is still an open question to train the vertically partitioned data efficiently and scalably by kernel methods   while keeping  data privacy.
% \begin{table*}[thbp]
% %	\small
% 	\center
% 	\caption{Representative federated learning  algorithms for vertically  partitioned data. () }
% %	%\vspace*{-8pt}
% 	\setlength{\tabcolsep}{2mm}
% 	\begin{tabular}{c|c|c|c|c|c|c}
% 		\hline
% 		\multirow{2}{*}{\textbf{Algorithm}}  &  \multirow{2}{*}{\textbf{Reference}} & \multicolumn{2}{c|}{\textbf{Properties}} & \multicolumn{2}{c|}{\textbf{Complexity}} & \multirow{2}{*}{\textbf{Statistical accuracy}}
% 		\\ \cline{3-4} \cline{5-6}
% 		&    &  \textbf{Historical information} &  \textbf{Stochastic} &   \textbf{Space } & \textbf{Time }  &   \\ \hline
% 		Online & \citet{wang2012generalization} & All historical samples  & Single  & $\mathcal{O}   (nd)$ & $\mathcal{O}   (n^2d)$ & Unknown \\
% 		Online & \citet{kar2013generalization} &   Fixed size samples & Single  & $\mathcal{O}   (sd )$ & $\mathcal{O}  (s n d)$ & Unknown\\ %\hline
% 		Online & \citet{boissier2016fast}   & 1th and 2th statistics  & Single & $\mathcal{O}( d^2)$& $\mathcal{O}  (n d^2)$  & Unknown \\
% 		\hline
% 		%D2SG & Our & General PL & General & General & $O   (d)$ & $O   (d)$ & $O   (\frac{1}{T}  )$  \\
% 		%D2SG & Our & General PL & General & General & $O   (Td)$ & $O   (Td)$ & $\mathcal{O}(\rho^T)$, $\rho\in (0,1)$  \\
% 		D2SG & Our  & All data & Double & $\mathcal{O}  (nd)$ & $\mathcal{O}  ( n d)$ & Guaranteed\\
% 		% & Our & General PL  & General & General & $O   (Td)$ & $O   (Td)$ &  \\
% 		% MSZOVR & Our &  Sequential  & Yes & Yes & Constant &   \\
% 		\hline
% 	\end{tabular}
% 	\label{table:methods}
% %\vspace*{-6pt}
% \end{table*}

To address this   challenging problem, in this paper,  we    propose a  new federated doubly stochastic kernel learning (FDSKL) algorithm  for vertically partitioned data. Specifically, we use random features to approximate the kernel mapping function  and use   doubly stochastic gradients to update kernel model, which  are all computed federatedly  without revealing the whole data  sample to each worker. %Note that the nonlinear prediction model is stored separately  in different workers.
We prove that FDSKL can converge to the optimal solution in $\mathcal{O}(1/t)$, where $t$ is the number of iterations. We also provide the analysis of the data security under the semi-honest assumption (\textit{i.e.}, Assumption \ref{ass_semi_honest}). Extensive experimental results on  a variety of benchmark datasets  show that FDSKL is significantly faster than state-of-the-art federated learning methods when dealing with kernels,  while  retaining the similar generalization performance.

\noindent \textbf{Contributions.} The  main contributions of this paper are summarized as
follows.
\begin{enumerate}[leftmargin=0.2in]
%\vspace*{-4pt}
\item Most of existing federated  learning algorithms on vertically  partitioned data  are limited to linearly separable model. Our FDSKL breaks the limitation of implicitly linear separability.
 %\vspace*{-4pt}
\item  To the best of our knowledge,  FDSKL  is the first   federated  kernel method which can train vertically  partitioned data \textit{efficiently and scalably}. We also prove that FDSKL can guarantee data security under the semi-honest assumption.
%\vspace*{-4pt}
\item  Existing doubly stochastic
kernel method is limited to the  diminishing learning rate. We extend it to  constant learning rate. Importantly, we prove that FDSKL with constant learning rate has a sublinear convergence rate.

%\vspace*{-4pt}
\end{enumerate}

%\noindent \textbf{Organization.} We organize the rest of the paper as follows.  Firstly, we  propose our FDSKL algorithm. Secondly, we provide the theoretical analysis of FDSKL. Thirdly, we show the experimental results. Finally, we give concluding remarks.

\section{Vertically Partitioned Federated  Kernel Learning Algorithm}\label{secFDSKL}
In this section, we first introduce the federated kernel learning problem, and then give a brief review of doubly stochastic kernel methods. Finally, we propose our FDSKL.

\subsection{Problem Statement}
 Given a training set $\mathcal{S}=\{(x_i,y_i)\}_{i=1}^{N}$, where $x_i \in \mathbb{R}^{d}$ and  $y_i\in \{+1,-1\}$  for  binary classification or $y_i\in \mathbb{R}$   for regression. Let $L(u,y )$ be a  scalar loss function which is convex with respect to $u \in \mathbb{R}$.
  We give several common loss functions in Table \ref{table:IP}.
  Given a positive definite  kernel function $K(x', x)$ and the associated  reproducing kernel Hilbert spaces (RKHS) $\mathcal{H}$ \citep{berlinet2011reproducing}. A kernel method tries to find a function $f \in \mathcal{H}$ \footnote{Because of the reproducing property \citep{zhang2009reproducing},
$\forall x \in \mathbb{R}^{d}$, $\forall f\in \mathcal{H}$, we always have $\langle f(\cdot),K(x,\cdot)\rangle_{\mathcal{H}}=f(x)$.}
for solving:
\begin{eqnarray}\label{objective_function}
\argmin_{f \in \mathcal{H}} \mathcal{R}(f)= \mathbb{E}_{(x,y) \in \mathcal{S}} L(f(x),y ) +\dfrac{\lambda}{2}\|  f \|_\mathcal{H}^2
\end{eqnarray}
where $\lambda > 0$ is a regularization parameter.

As mentioned previously, in a lot of real-world data mining and machine learning applications,  the input of training sample $(x,y)$ is partitioned vertically into
$q$ parts, \emph{i.e.}, $x = [x_{\mathcal{G}_1},x_{\mathcal{G}_2},\ldots,x_{\mathcal{G}_q}]$, and $x_{\mathcal{G}_\ell} \in \mathbb{R}^{d_\ell}$ is
stored on the $\ell$-th worker and $\sum_{\ell=1}^q d_\ell=d$.  According to whether the label is included in a worker, we divide the workers into two types: one is active worker and the other is passive worker, where the active worker is the data provider who holds the label of a sample, and the passive worker only has the input of a sample. The active worker would be a dominating server in federated learning, while passive workers play the role of clients \citep{cheng2019secureboost}.
% \begin{definition}[Active  and passive workers]\label{def_active_passive}
% We define the active worker as the data provider who holds
%  the label of a sample.  We define the data provider which has only the input of a sample
% as a passive worker.
% \end{definition}
\noindent We let  $D^\ell$  denote the data stored on the $\ell$-th worker, where the  labels $y_i$ are distributed
on active workers.
$D^\ell$ includes parts of labels $\{y_i\}_{i=1}^l$.
Thus, our goal in this paper can be presented as follows.

%\vspace*{2pt}
 \begin{mdframed}
\begin{myitemize}
\noindent \textbf{Goal:} Make active workers to cooperate   with passive workers to solve the nonlinear learning problem (\ref{objective_function}) on the vertically partitioned data $\{ D^\ell\}_{\ell=1}^q$   while keeping the vertically partitioned data private.
\end{myitemize}
\end{mdframed}

\noindent \textbf{Principle of  FDSKL.} Rather than dividing a kernel into several sub-matrices, which requires expensive kernel merging operation, FDSKL uses the random feature approximation to achieve efficient computation parallelism under the federated learning setting. The efficiency of FDSKL is owing to the fact that the computation of random features can be linearly separable.  Doubly stochastic gradient (DSG) \cite{gu2018asynchronous,gu2019scalable} is a scalable and efficient kernel method \citep{dai2014scalable,xie2015scale,gu2018asynchronous1} which uses the doubly stochastic gradients \textit{w.r.t.} samples and random features to update the kernel function. We extend DSG to  the vertically partitioned data, and propose  FDSKL. Specifically, each FDSKL worker computes a partition of the random features using only the partition of data it keeps. When computing the global functional gradient of the kernel, we could efficiently reconstruct the entire random features from the local workers.
Note that although FDSKL is also a privacy-preservation gradient
descent algorithm for vertically partitioned data,  FDSKL breaks the limitation of implicit linear separability used in the existing privacy-preservation federated learning algorithms as discussed previously.

\subsection{Brief Review of Doubly Stochastic Kernel Methods}
We first introduce  the technique of random feature approximation, and then give a brief review of DSG algorithm.
\subsubsection{Random Feature Approximation}
%`Kernel trick' make it possible for learning a nonlinear model through kernel function. Kernel function, $K(x,x') = \langle\phi_0(x),\phi_0(x')\rangle$, implicitly map input vector to a high dimensions (even infinite dimensions) reproducing kernel Hilbert space (RKHS) by computing the inner product of the mapped features. However, as mentioned previously, a $O(n^2)$ size kernel matrix need to be stored and computed.
%Random feature approximation based on a simple but solid idea that approximates the kernel function using $m$ explicit random Fourier features directly. So that the random feature matrix can be computed in time $O(nmd)$ using $O(nm)$ memory. Subsequent algorithms then only operate on an $O(nm)$ matrix.
%To address this problem,

Random feature \citep{rahimi2008random,rahimi2009weighted,gu2018asynchronous,geng2019scalable,shi2019quadruply,DBLP:ShiG0H20} is a powerful technique to make  kernel methods scalable. It uses the intriguing duality between  positive definite kernels which are continuous and shift
invariant  (\emph{i.e.}, $K(x,x')=K(x-x')$) and stochastic processes as shown in Theorem \ref{thmBochner}.
\begin{theorem}[\cite{rudin1962fourier}]\label{thmBochner} A continuous, real-valued, symmetric and shift-invariant function $K(x,x')=K(x-x')$ on $\mathbb{R}^{d}$ is a positive definite kernel if and only if there is a finite non-negative measure $\mathbb{P}(\omega)$ on $\mathbb{R}^{d}$, such that
\begin{align}
 K(x-x')=\int_{\mathbb{R}^{d}} e^{i \omega^T (x-x')} d\mathbb{P}(\omega)
 =\int_{\mathbb{R}^{d} \times [0,2\pi]} 2 \cos (\omega^T x+b)\cos (\omega^T x'+b) d(\mathbb{P}(\omega) \times \mathbb{P}(b))
%\vspace{-5pt}
\end{align}
 where $\mathbb{P}(b)$ is  a uniform distribution on $[0,2\pi]$, and  $\phi_\omega(x)=\sqrt{2} \cos (\omega^T x+b)$.
%Specifically, according to the Bochner theorem \cite{wendland2004scattered},
%for any PD kernel $K(\cdot,\cdot)$, there exists a set $\Omega$, a probability measure $\mathbb{P}$ and a random feature map $\phi_\omega(x)$, such that
%\begin{eqnarray}
%K(x,x')=\int_{\Omega}\phi_\omega(x) \phi_\omega(x')d\mathbb{P}(\omega)
%\end{eqnarray}
\end{theorem}
According to Theorem \ref{thmBochner}, the value of the kernel function can be approximated by explicitly computing the random feature maps $\phi_\omega(x)$ as follows.
\begin{eqnarray}
K(x,x')\approx \frac{1}{m}\sum_{i=1}^{m}\phi_{\omega_i}(x) \phi_{\omega_i}(x')
\end{eqnarray}
where $m$ is the number of random features and $\omega_i$ are drawn from $\mathbb{P}(\omega)$.
Specifically, for Gaussian RBF kernel  $K(x, x')=\exp(-||x-x'||^2/2\sigma^2)$, $\mathbb{P}(\omega)$ is a Gaussian distribution
with density proportional to $\exp(-\sigma^2 \| \omega \|^2/2)$. For the Laplace kernel \citep{yang2014random}, this yields a Cauchy
distribution. Notice that the computation of a random feature map $\phi$ requires to compute a linear combination of the raw input features, which can also be partitioned vertically. This property makes random feature approximation well-suited for the federated learning setting. %For the Martern kernel \citep{smola1998learning}, this yields the convolutions of the unit ball. Dai et al. summarized most of the important kernels with their associated densities and explicit features in \cite{dai2014scalable}.
%\vspace*{-8pt}
\subsubsection{Doubly Stochastic Gradient}

Because the functional gradient in RKHS $\mathcal{H}$ can be computed as $\nabla f(x)= K(x, \cdot)$, the
stochastic  gradient of $\nabla f(x)$ \textit{w.r.t.} the random feature $\omega$ can be denoted by (\ref{eq3}).
\begin{equation}\label{eq3}
\xi(\cdot)= \phi_{\omega}(x)\phi_{\omega}(\cdot)
\end{equation}
%Note that we have $\mathbb{E}_{\omega} \xi(\cdot) = \nabla f(x)$.
 Given a randomly sampled data instance $(x,y)$, and a random feature $\omega$, the doubly stochastic gradient of the loss function $L(f(x_i),y_i)$ on RKHS \textit{w.r.t.} the sampled instance $(x,y)$ and the random direction $\omega$ can be formulated  as follows.
\begin{equation}
\label{def_doublySG}
\zeta(\cdot)= L'(f(x_i),y_i)\phi_{\omega}(x_i)\phi_{\omega}(\cdot)
\end{equation}
%Note that we have $\mathbb{E}_{(x,y)}\mathbb{E}_{\omega} \zeta(\cdot) = \nabla \mathbb{E}_{(x,y) \in \mathcal{S}} L(f(x),y ))$.
Because $\nabla ||f||^2_{\mathcal{H}} = 2f$, the  stochastic gradient of $\mathcal{R}(f)$ can be formulated as follows.
\begin{eqnarray}
\label{def_doublySG2}
 \widehat{\zeta}(\cdot) =\zeta(\cdot) + \lambda f(\cdot)
= L'(f(x_i),y_i)\phi_{\omega_i}(x_i)\phi_{\omega_i}(\cdot) + \lambda f(\cdot)
\end{eqnarray}
Note that we have  $\mathbb{E}_{(x,y)}\mathbb{E}_{\omega} \widehat{\zeta}(\cdot) =  \nabla \mathcal{R}(f)$.
According to the stochastic gradient (\ref{def_doublySG2}), we can update the solution by stepsize $\gamma_t$. Then, let $f_{1}(\cdot)=\textbf{0}$, we have that
\begin{align}\label{eqf}
f_{t+1}(\cdot)
=& f_{t}(\cdot) - \gamma_t \left ( \zeta(\cdot) + \lambda f(\cdot) \right )=\sum_{i=1}^t - \gamma_i \prod_{j=i+1}^t (1-\gamma_j \lambda ) \zeta_i(\cdot)
\\ \nonumber =& \sum_{i=1}^t \underbrace{- \gamma_i \prod_{j=i+1}^t (1-\gamma_j \lambda ) L'(f(x_i),y_i)\phi_{\omega_i}(x_i)}_{\alpha_i^t}\phi_{\omega_i}(\cdot)
%\vspace{-10pt}
\end{align}
From (\ref{eqf}),  $\alpha_i^t$ are the important coefficients which defines the model of $f(\cdot)$. Note that the model $f(x)$ in (\ref{eqf})  do not satisfy the assumption of implicitly linear separability same to the kernel model $f(x)=\sum_{i}^N \alpha_i K(x_i,x)$ as mentioned in the first section.

\subsection{Federated Doubly Stochastic Kernel Learning Algorithm}
We first present the system structure of FDSKL, and then give a  detailed description of  FDSKL.
\subsubsection{System Structure}
As mentioned before, FDSKL is a federated learning algorithm where each worker keeps its local vertically partitioned  data.  Figure \ref{structureFDSKL} presents the  system structure of FDSKL. The main idea behind FDSKL's parallelism is to vertically divide the computation of the random features.  Specifically, we  give  detailed descriptions of data privacy, model privacy and tree-structured
communication, respectively, as follows.
\begin{figure*}[h]
%\vspace*{-4pt}
\center
\includegraphics[scale=0.34]{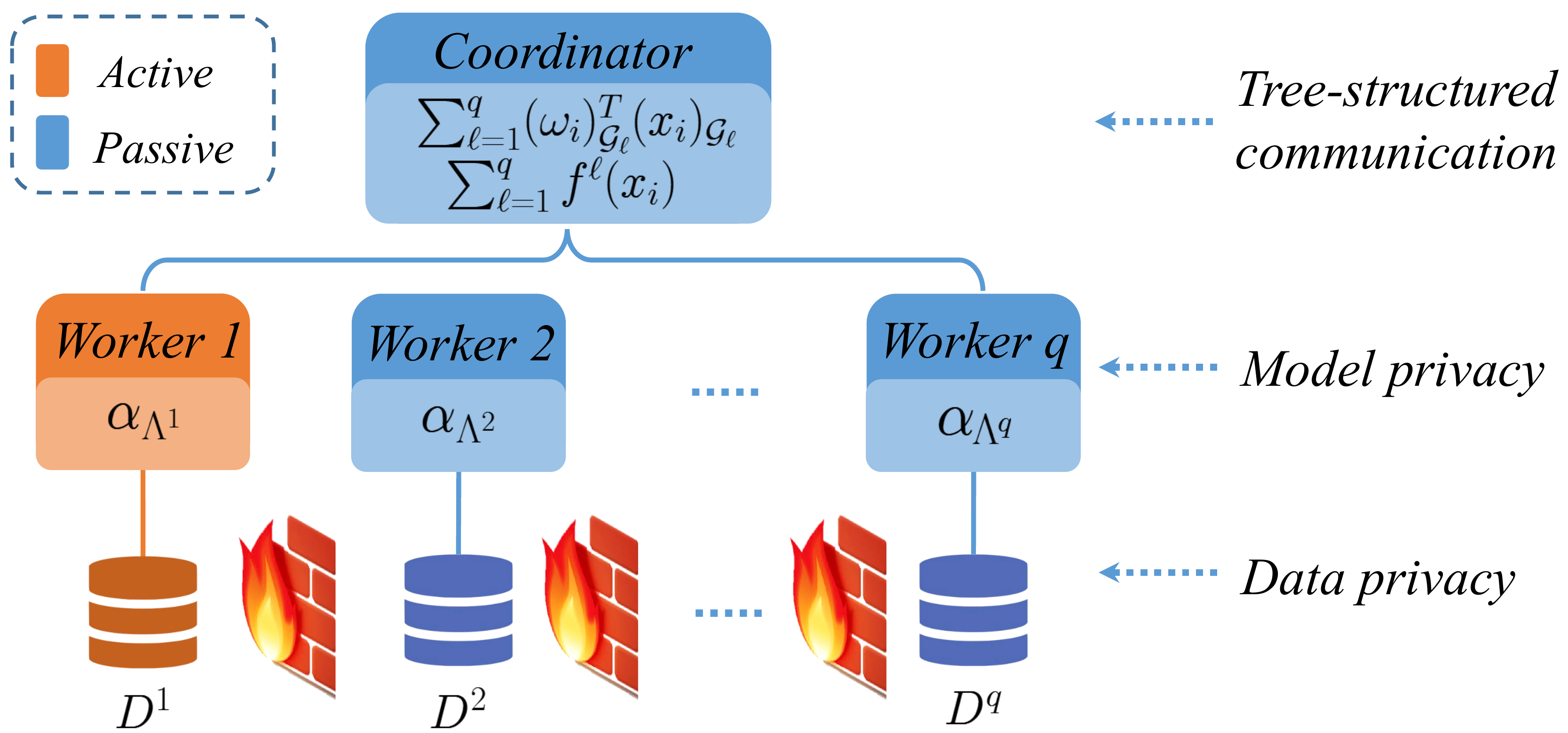}
% \vspace*{-20pt}
 \caption{System structure of FDSKL.}
 \label{structureFDSKL}
% \vspace*{-10pt}
\end{figure*}

\begin{figure*}[htbp!]
% \vspace*{-12pt}
	\begin{subfigure}[b]{0.45\textwidth}
\centering
%\hspace*{-0.9cm}
		\includegraphics[height=2.2in]{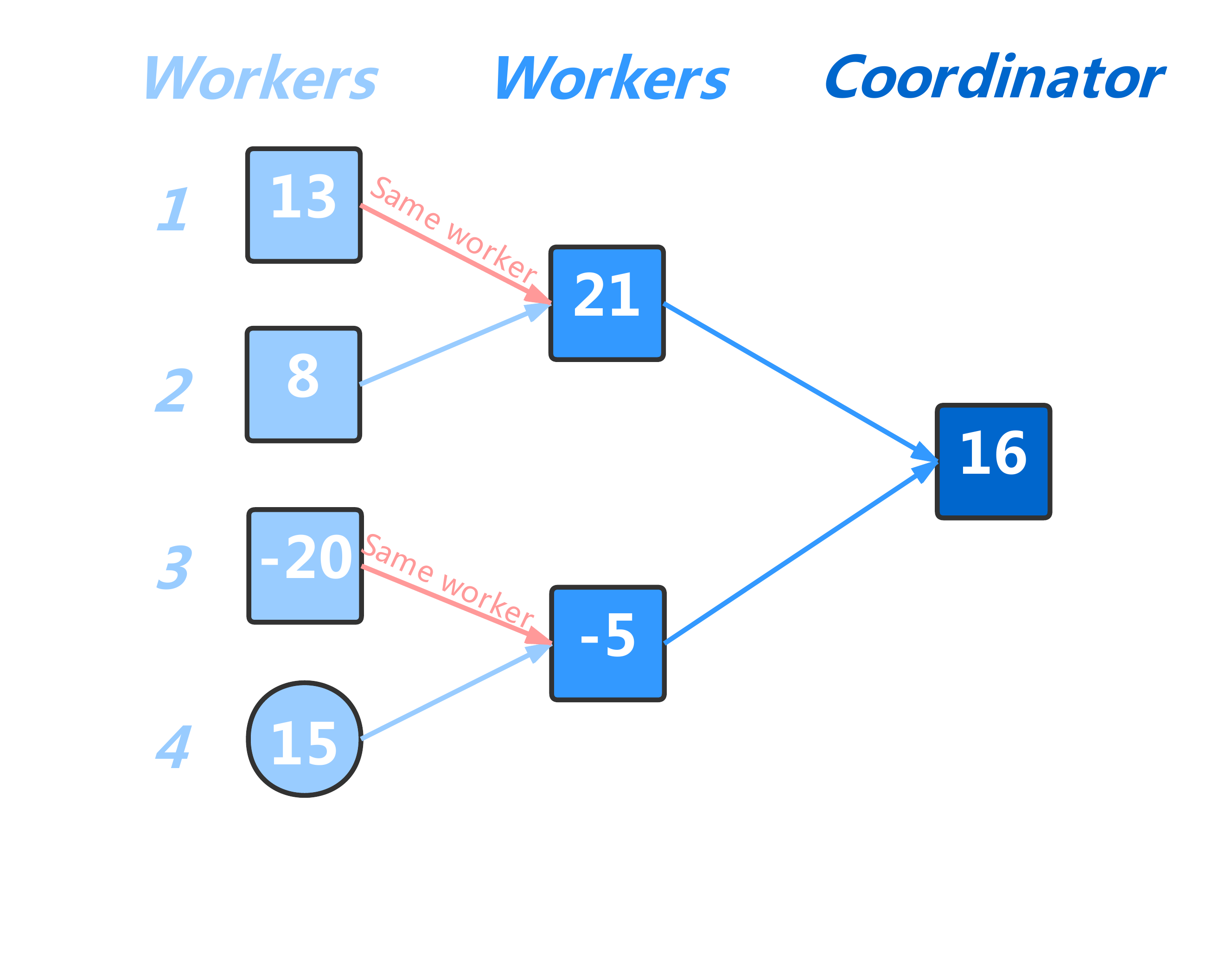}
%\vspace*{-26pt}
		\caption{Tree structure $T_1$ on workers $\{1,\ldots,4\}$} \label{FigureTSC1}
	\end{subfigure}
	\begin{subfigure}[b]{0.45\textwidth}
\centering
		\includegraphics[height=2.2in]{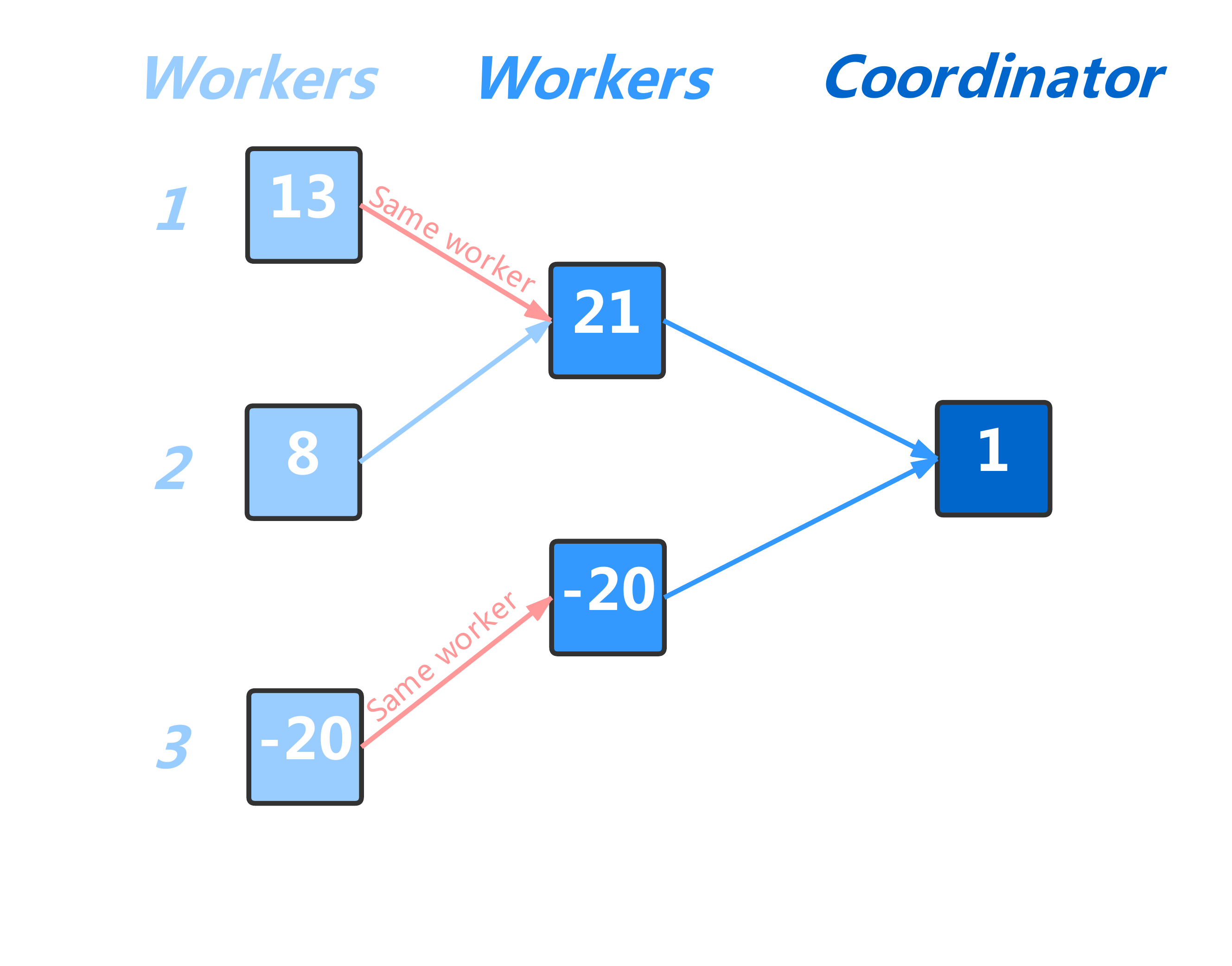}
%\vspace*{-26pt}
		\caption{Tree structure $T_2$ on workers $\{1,\ldots,3\}$} \label{FigureTSC2}
	\end{subfigure}
%\vspace*{-8pt}
\caption{Illustration of  tree-structured communication with two totally different tree structures $T_1$ and $T_2$.}
\label{FigureTSC}
\end{figure*}
\begin{enumerate}[leftmargin=0.2in]
%\vspace{-8pt}
\item
\textbf{Data  Privacy:} \  To keep the vertically partitioned  data privacy,  we need  to divide the computation of  the value of   $\phi_{\omega_i}(x_i)=\sqrt{2} \cos (\omega_i^T x_i+b)$ to avoid transferring the local data $(x_i)_{\mathcal{G}_\ell}$ to other workers. Specifically, we send a random seed to  the $\ell$-th worker. Once the $\ell$-th worker receive the random seed, it can generate the random direction $\omega_i$ uniquely according to the random seed. Thus, we can locally compute $(\omega_i)_{\mathcal{G}_\ell}^T (x_i)_{\mathcal{G}_\ell}+b$ which avoids directly transferring the local data $(x_i)_{\mathcal{G}_\ell}$ to other workers for computing $\omega_i^T x_i+b$.  In the next section, we will discuss  it is hard to infer any $(x_i)_{\mathcal{G}_\ell}$  according to  the value of $(\omega_i)_{\mathcal{G}_\ell}^T (x_i)_{\mathcal{G}_\ell}+b$ from other workers.

%\vspace{-2pt}
\item \textbf{Model Privacy:} \ In addition to keep the vertically partitioned  data privacy, we also keep the model privacy. Specifically, the model coefficients $\alpha_i$ are stored in different workers separately and  privately. According to the location of the  model coefficients $\alpha_i$, we partition the model coefficients $\{ \alpha_i \}_{i=1}^T$ as   $\{\alpha_{\Lambda^\ell} \}_{\ell=1}^q$, where $\alpha_{\Lambda^\ell} $ denotes the  model coefficients at the $\ell$-th worker, and $\Lambda^\ell$ is the set of corresponding iteration indices. We do not directly transfer the local model coefficients $\alpha_{\Lambda^\ell} $ to other workers.
To compute $f(x)$, we locally compute $f^\ell(x)=\sum_{i\in \Lambda^\ell} \alpha_i\phi_{\omega_i}(x)$ and transfer it to other worker, and $f(x)$ can be reconstructed by summing over all the $f^\ell(x)$. It is difficult to infer the the local model coefficients $\alpha_{\Lambda^\ell} $ based on the value of  $f^\ell(x)$ if $|\Lambda^\ell|\geq 2$. Thus, we achieve the model privacy.
%\vspace{-2pt}
\item \textbf{Tree-Structured Communication:} \ In order to   obtain $\omega_i^T x_i$ and $f(x_i)$, we need to accumulate the local results from different workers. Zhang et al. \citeyearpar{zhang2018feature} proposed an efficient  tree-structured communication  scheme to
get the global sum which is faster than the simple strategies  of star-structured communication \citep{wan2007privacy} and  ring-structured communication \citep{yu2006privacy}. Take 4 workers as an example, we pair the workers
so that while worker 1 adds the result from worker 2,
worker 3 can add the result from worker 4 simultaneously. Finally, the results from  the two pairs of workers are sent to the coordinator and we  obtain the global sum (please see Figure \ref{FigureTSC1}).
 If the above procedure is in a reverse order,  we call it  a reverse-order tree-structured communication.
 Note that both of the tree-structured communication  scheme and its reverse-order scheme are synchronous procedures.
%\vspace{-6pt}
\end{enumerate}

\subsubsection{Algorithm}
To extend DSG to federated  learning on vertically partitioned data while keeping data privacy, we need to carefully design the procedures of computing $\omega_i^T x_i+b$, $f(x_i)$ and the solution updates, which are presented in detail as follows.

\begin{enumerate}[leftmargin=0.2in]
%\vspace{-6pt}
\item
\textbf{Computing $\omega_i^T x_i+b$:} \  %As mentioned  the sample $x_i$ is partitioned into $q$ parts,  we cannot directly compute $\omega_i^T x_i$ due to the.
  We generate the random direction $\omega_i$ according to  a same random seed $i$ and a probability measure $\mathbb{P}$ for each worker. Thus, we can locally compute $(\omega_i)_{\mathcal{G}_\ell}^T (x_i)_{\mathcal{G}_\ell}$.  To  keep $(x_i)_{\mathcal{G}_\ell}$ private, instead of directly transferring $(\omega_i)_{\mathcal{G}_\ell}^T (x_i)_{\mathcal{G}_\ell}$ to other workers, we randomly generate $b^\ell$ uniformly from $[0, 2 \pi]$, and transfer $(\omega_i)_{\mathcal{G}_\ell}^T (x_i)_{\mathcal{G}_\ell}+b^\ell $ to another worker. After all workers have calculated $(\omega_i)_{\mathcal{G}_\ell}^T (x_i)_{\mathcal{G}_\ell}+b^\ell $ locally,  we can
get the global sum $\sum_{\hat{\ell}=1}^q  \left ((\omega_i)_{\mathcal{G}_{\hat{\ell}}}^T (x_i)_{\mathcal{G}_{\hat{\ell}}}+b^{\hat{\ell}} \right )$   efficiently and  safely by using a tree-structured communication  scheme based on the tree structure $T_1$ for workers $\{1, \ldots,q\}$ \citep{zhang2018feature}.

Currently, for the   $\ell$-th worker, we get multiple values of $b$ with $q$ times. To recover the value of $\sum_{\hat{\ell}=1}^q  \left ((\omega_i)_{\mathcal{G}_{\hat{\ell}}}^T (x_i)_{\mathcal{G}_{\hat{\ell}}} \right )+b$, we pick up one $b^{\ell'}$ from $\{1,\ldots,  q\}-\{ \ell \}$ as the value of $b$  by removing other  values of $b^\ell$ (\emph{i.e.},  removing $\overline{b}^{\ell'}=\sum_{{{\hat{\ell}}} \neq \ell'} b^{{{\hat{\ell}}}}$). In order to prevent leaking any information of $b^\ell$, we use a \textit{totally different tree structure $T_2$} for workers  $\{1, \ldots,q\}-\{\ell'  \}$  (please see Definition \ref{def_total_differ_tree} and Figure \ref{FigureTSC}) to compute $\overline{b}^{\ell'}=\sum_{{{\hat{\ell}}} \neq \ell'} b^{{{\hat{\ell}}}}$.
The detailed  procedure of computing $\omega_i^T x_i+b$ is summarized in Algorithm \ref{protocol1}.
\begin{definition}[Two totally different tree structures]\label{def_total_differ_tree} For two tree structures $T_1$ and $T_2$, they are  totally different if there does not exist a subtree with more than one leaf which belongs to both of $T_1$ and $T_2$.
\end{definition}

%\vspace{-3pt}
\item \textbf{Computing $f(x_i)$:} \ According to (\ref{eqf}), we have that $f(x_i) =\sum_{i=1}^t \alpha_i^t \phi_{\omega_i}(x_i)$. However, $\alpha_i^t$ and $\phi_{\omega_i}(x_i)$ are stored in different workers. Thus, we first locally compute $f^\ell(x_i)= \sum_{i \in \Lambda^\ell} \alpha_i^{t} \phi_{\omega_i}(x_i)$ which is summarized in Algorithm \ref{alg1}.  By using a tree-structured communication  scheme \citep{zhang2018feature}, we can
get the global sum $\sum_{\ell=1}^q f^{\ell}(x_i)$  efficiently which is equal to $f(x_i)$ (please see Line \ref{step7} in Algorithm \ref{algorithm3}).

%\vspace{-3pt}
\item \textbf{Updating Rules:} \ Because $\alpha_i^t$ are stored in different workers, we use a  communication  scheme \citep{zhang2018feature} with a reverse-order tree structure to update $\alpha_i^t$  in each workers by the coefficient  $(1-\gamma \lambda)$ (please see Line \ref{step10} in Algorithm \ref{algorithm3}).
%\vspace{-6pt}
\end{enumerate}

Based on these key procedures, we summarize  our FDSKL algorithm in  Algorithm \ref{algorithm3}. %Specifically, FDSKL   repeats the following  steps based on the  rules of computing $\omega_i^T x_i$, $f(x_i)$ and updating the solution as discussed above.
%\begin{enumerate}[leftmargin=0.2in]
%
%\item \textit{Select random sample (line 1 in Algorithm \ref{algorithm3}):} \ Randomly sample a sample $(x_i)_{\mathcal{G}_\ell}$ from the local data $D^\ell$.
%\item \textit{Sample random feature (line 3 in Algorithm \ref{algorithm3}):} \  Sample random feature $\omega_i \sim \mathbb{P}(\omega)$ according the given random  seed $i$.
%\item \textit{Compute $\omega_i^T x_i+b$ (line 6 in Algorithm \ref{algorithm3}):} \  Compute $\omega_i^T x_i+b$ using tree-structured communication scheme.
%\item \textit{Compute the prediction value $f(x_i)$ (line 6 in Algorithm \ref{algorithm3}):} \  Use tree-structured communication scheme to compute $f(x_i)=\sum_{\ell=1}^q f^\ell(x_i)$.
%\item \textit{Compute  the current coefficient $\alpha_i$  (line 9 in Algorithm \ref{algorithm3}):} \ We compute the the current coefficient $\alpha_i$ .
%
%\item \textit{Update the former coefficients (line 10 in Algorithm \ref{algorithm3}):} \  We update the former coefficients $\alpha_j$ for $j=1,\cdots,i-1$ according to the update rule (\ref{eqf}).
%
%\end{enumerate}
%Note that  we use the synchronous procedures for the lines \ref{step7} and \ref{step10}  of Algorithm \ref{algorithm3} to ensure the doubly stochastic gradients are equivalent to the sequential DSG.
Different to the diminishing learning rate used in DSG, our FDSKL uses a constant learning rate $\gamma$ which can be  implemented more easily  in  the parallel computing environment. However, the convergence analysis for  constant learning rate is more difficult than the one for diminishing learning rate. We give the theoretic analysis in the following section.

\begin{algorithm}[ht]
\renewcommand{\algorithmicrequire}{\textbf{Input:}}
\renewcommand{\algorithmicensure}{\textbf{Output:}}
\renewcommand{\algorithmicloop}{\textbf{keep doing in parallel}}
\renewcommand{\algorithmicendloop}{\textbf{end parallel loop}}
\caption{Vertically Partitioned Federated Kernel Learning Algorithm  (FDSKL) on the  $\ell$-th active worker}
	\begin{algorithmic}[1] %[1] enables line numbers
		\REQUIRE $\mathbb{P}(\omega)$, local normalized data  $D^\ell$,  regularization parameter $ \lambda$, constant learning rate $\gamma$.
 \LOOP
 %		\FOR{$i=1,...,t^{\ell}$}
		\STATE \label{step2} Pick up an instance $(x_i)_{\mathcal{G}_\ell}$ from the local data $D^\ell$ with index $i$.
		\STATE \label{step3} Send $i$ to other workers  using a reverse-order tree structure  $T_0$.
\STATE \label{step4} Sample $\omega_i \sim \mathbb{P}(\omega)$ with the random  seed $i$ for all workers.
%    \STATE Compute $(\omega_i)_{\mathcal{G}_\ell}^T (x_i)_{\mathcal{G}_\ell} $.

   \STATE \label{step5}  Use Algorithm \ref{protocol1}  to compute $\omega_i^T x_i+b $ and locally save it.
   \STATE \label{step6} Compute  $f^{\ell'}(x_i)$ for $\ell'=1,\ldots,q$ by calling Algorithm \ref{alg1}.
		\STATE \label{step7} Use tree-structured communication scheme based on $T_0$ to compute $f(x_i)=\sum_{\ell=1}^q f^\ell(x_i)$.
 \STATE \label{step8}  Compute $\phi_{\omega_i}(x_i)$ according to $\omega_i^T x_i+b$.
\STATE \label{step9} Compute $\alpha_i = -\gamma \left ( L'(f(x_i),y_i)\phi_{\omega_i}(x_i) \right )$ and locally save $\alpha_i$.

\STATE \label{step10}  Update $\alpha_j = (1-\gamma \lambda)\alpha_j$ for all previous $j$  in  the $\ell$-th worker  and other workers. % using a reverse-order tree structure.
%\STATE \textbf{Opt 2 (multiple copies of $\alpha$):} Broadcast $\alpha_i$ to other workers, and update $\alpha_j = (1-\gamma_j\lambda)\alpha_j$ for all previous $j$ using a reverse-order tree structure.
\ENDLOOP
%		\ENDFOR
\ENSURE $\alpha_{\Lambda^\ell}$.
\end{algorithmic}
\label{algorithm3}
\end{algorithm}

%\begin{algorithm}[!ht]	
%	\caption{$\{\alpha_i\}_{i=1}^t$ =\textbf{ QSG-S2AUC}$(D_p,D_n,p(x))$} \label{alg1}
%
%	\label{alg:train}
%\end{algorithm}

\begin{algorithm}[!ht]
	\caption{Compute $f^\ell(x)$ on the $\ell$-th active worker} \label{alg1}
	\renewcommand{\algorithmicrequire}{\textbf{Input:}}
	\renewcommand{\algorithmicensure}{\textbf{Output:}}
	\begin{algorithmic}[1] %[1] enables line numbers
		\REQUIRE $\mathbb{P}(\omega)$, $\alpha_{\Lambda^\ell}$,  $\Lambda^\ell$, $x$.

		\STATE \label{step1} Set $f^\ell(x) = 0$.
		\FOR{each $i\in \Lambda^\ell$} \label{step2}
%		\STATE Sample $\omega_i \sim p(\omega)$ with seed $i$.
 \STATE \label{step3} Sample $\omega_i \sim \mathbb{P}(\omega)$ with the random  seed $i$ for all workers.
%    \STATE Compute $(\omega_i)_{\mathcal{G}_\ell}^T (x)_{\mathcal{G}_\ell} $.

   \STATE \label{step4} Obtain $\omega_i^T x+b $ if it is locally saved, otherwise   compute $\omega_i^T x+b $ by using Algorithm \ref{protocol1}.

 \STATE  \label{step5} Compute $\phi_{\omega_i}(x)$ according to $\omega_i^T x+b$.
		\STATE \label{step6} $f^\ell(x)=f^\ell(x)+\alpha_i\phi_{\omega_i}(x)$
		\ENDFOR \label{step7}
		\ENSURE $f^\ell(x)$
	\end{algorithmic}
	\label{alg:predict}
\end{algorithm}

\begin{algorithm}[!ht]
%\floatname{algorithm}{{Protocol}
%\renewcommand{\thealgorithm}{}
\renewcommand{\algorithmicrequire}{\textbf{Input:}}
\renewcommand{\algorithmicensure}{\textbf{Output:}}
\caption{Compute $\omega_i^T x_i+b $  on the   $\ell$-th active worker} \label{protocol1}
\begin{algorithmic}[1]
\REQUIRE $\omega_i$, $x_i$
\\
\COMMENT{// This loop asks multiple workers running in parallel.}
\FOR{$\hat{\ell}=1,\ldots,q$} \label{step1}
\STATE \label{step2} Compute $(\omega_i)_{\mathcal{G}_{\hat{\ell}}}^T (x_i)_{\mathcal{G}_{\hat{\ell}}} $ and randomly generate a uniform number $b^{\hat{\ell}}$ from $[0, 2\pi]$ with the seed $\sigma^{\hat{\ell}}(i)$.
\STATE \label{step3} Calculate $(\omega_i)_{\mathcal{G}_{\hat{\ell}}}^T (x_i)_{\mathcal{G}_{\hat{\ell}}} + b^{\hat{\ell}}$.
\ENDFOR \label{step4}
\STATE \label{step5} Use  tree-structured communication scheme based on the tree structure $T_1$ for workers $\{1, \ldots,q\}$ to compute $\xi=\sum_{\hat{\ell}=1}^q \left ( (\omega_i)_{\mathcal{G}_{\hat{\ell}}}^T (x_i)_{\mathcal{G}_{\hat{\ell}}}+b^{\hat{\ell}} \right )$.
\STATE \label{step6} Pick up $\ell' \in \{ 1,\ldots,q\} -\{\ell \} $  uniformly at random.
\STATE \label{step7} Use  tree-structured communication scheme based on the totally  different tree structure $T_2$ for workers  $\{1, \ldots,q\}-\{\ell'  \}$  to compute $\overline{b}^{\ell'}=\sum_{{\hat{\ell}} \neq \ell'} b^{{\hat{\ell}}}$.
\ENSURE $\xi-\overline{b}^{\ell'}$.
\end{algorithmic}
\end{algorithm}

%\begin{protocol}{Secure Multi-Party Shuffling Scheme}
%\textit{Inputs.} For all $i \in [n]$, party~$P_i$ holds an input~$x_i$.
%  Let~$C$ denote the probabilistic sorting network of~[LP90] and~$d$ denote its depth.
%\sbline
%\textit{Goal.} Parties jointly compute a random shuffle of their inputs.
%\sbline
%\textit{The protocol:}
%\begin{enumerate}
%  \item \textbf{Setup.}
%  \begin{enumerate}
%    \item
%    Parties run \textsf{Build-Quorums} to agree on $n$ quorums $Q_1, \ldots, Q_n$.
%
%    \item
%    Parties in $Q_1$ run \textsf{Gen-Rand} and \textsf{VSS-Reconst} repeatedly to generate
%    a sequence~$R$ of $\Theta(\log^2 n)$ random bits.
%  \end{enumerate}
%\end{enumerate}
%\end{protocol}

\section{Theoretical Analysis}\label{secTheoAna}
In this section, we  provide the convergence, security and complexity analyses to FDSKL.
%We first give several basic  assumptions, and then  provide  the convergence rate  analysis of   FDSKL.
\subsection{Convergence Analysis}
As the basis of our analysis, our first lemma states that the output of Algorithm \ref{protocol1} is actually equal to  $\omega_i^T x+b$. The proofs are provided in  Appendix.
\begin{lemma}\label{lemma1}
 The output  of Algorithm \ref{protocol1} (\emph{i.e.}, $ \sum_{\hat{\ell}=1}^q  \left ( (\omega_i)_{\mathcal{G}_{\hat{\ell}}}^T (x)_{\mathcal{G}_{\hat{\ell}}}  +b^{\hat{\ell}} \right ) - \overline{b}^{\ell'}$  is equal to $\omega_i^T x+b$,  where each $b^{\hat{\ell}}$ and $b$ are drawn from a uniform distribution on $[0,2\pi]$, $\overline{b}^{\ell'}=\sum_{{\hat{\ell}} \neq \ell'} b^{{\hat{\ell}}}$, and $\ell' \in \{ 1,\ldots,q\} -\{\ell \} $.
\end{lemma}

Based on Lemma \ref{lemma1}, we can conclude that the federated learning algorithm (\emph{i.e.}, FDSKL)  can produce the same doubly stochastic gradients as that of a DSG algorithm with constant learning rate.  Thus, under Assumption \ref{ass1}, we  can prove that FDSKL converges to the optimal solution almost at a rate of $\mathcal{O}(1/t)$ as shown in Theorem \ref{thm1}.  The proof is provided in Appendix, Note that the convergence proof of the original DSG algorithm in \citep{dai2014scalable} is  limited to diminishing learning rate.

\begin{assumption} \label{ass1} Suppose the following conditions hold.
\begin{enumerate}[leftmargin=0.2in]
%\vspace{-3pt}
\item There exists an optimal solution, denoted as $f_*$, to the problem  (\ref{objective_function}).
%\vspace{-1pt}
\item We have an upper bound for the derivative of $L(u,y)$ w.r.t. its 1st argument, \emph{i.e.},  $|L'(u,y)|<M$.
%\vspace{-1pt}
\item
The loss function $L(u,y)$ and its first-order derivative are  $\mathcal{L}$-Lipschitz continuous in terms of the first argument.
%\vspace{-1pt}
\item We have an upper bound $\kappa$ for the kernel value, \emph{i.e.}, $K(x,x') \le \kappa$. We have an upper bound $\phi$ for random feature mapping, \emph{i.e.}, $|\phi_{\omega}(x)\phi_{\omega}(x')| \le \phi$.
%\vspace{-3pt}
\end{enumerate}
\end{assumption}
\begin{theorem}\label{thm1}
	Set  $\epsilon>0$,  $\min \{\frac{1}{\lambda}, \frac{\epsilon \lambda}{4M^2(\sqrt{\kappa}+\sqrt{\phi})^2}\}>\gamma > 0$, for Algorithm \ref{algorithm3}, with $\gamma=\frac{\epsilon\vartheta}{8\kappa B}$ for $\vartheta\in\left( \right.0,1\left.\right]$, under Assumption \ref{ass1}, we will reach
$
\mathbb{E} \Big[|f_{t}(x) - f_*(x)|^2\Big] \le \epsilon
$ after
\begin{equation}\label{tsolu}
t \geq \frac{8\kappa B\log (8\kappa e_1/\epsilon)}{\vartheta\epsilon \lambda}
\end{equation} iterations, where $ B=\left[\sqrt{G_2^2+G_1}+G_2\right]^2$, $G_1=\frac{2\kappa M^2}{\lambda}$, $G_2=\frac{\kappa^{1/2}M(\sqrt{\kappa}+\sqrt{\phi})}{2\lambda^{3/2}}$ and $e_1=\mathbb{E} [\|h_1-f_*\|_{\mathcal{H}}^2]$.
\end{theorem}
\begin{remark}
	Based on Theorem \ref{thm1}, we have  that for any given data $x$, the evaluated value of $f_{t+1}$ at $x$ will converge to that of a solution close to $f_*$ in terms of the Euclidean distance. The rate of convergence is almost $\mathcal{O}(1/t)$, if eliminating the $\log(1/\epsilon)$ factor. Even though our algorithm has included more randomness by using random features, this rate is nearly the same as standard SGD. As a result, this guarantees the efficiency of the proposed algorithm.
\end{remark}
  \begin{table*}[htbp]
%\vspace*{-4pt}
\small
 \center
 \caption{The  state-of-the-art kernel methods compared in our experiments. (BC=binary classification,
R=regression) }
%\vspace*{-8pt}
 \setlength{\tabcolsep}{0.5mm}
\begin{tabular}{c|c|c|c|c}
\hline
 \textbf{Algorithm}  &  \textbf{Reference} &   \textbf{Problem} & \textbf{Vertically Partitioned Data} &  \textbf{Doubly Stochastic}    \\ \hline
LIBSVM & \cite{CC01a}  & BC+R &  No  & No  \\
DSG & \cite{dai2014scalable}  &    BC+R & No & Yes  \\
PP-SVMV & \cite{yu2006privacy} &  BC+R &  Yes & No  \\
\hline
FDSKL & Our & BC+R & Yes  & Yes  \\
% MSZOVR & Our &  Sequential  & Yes & Yes & Constant &   \\
 \hline
\end{tabular}
\label{table:methods}
%\vspace*{-8pt}
\end{table*}
\subsection{Security Analysis}
% The   semi-honest assumption (please refer to Assumption \ref{ass_semi_honest})  is commonly used in security analysis \citep{wan2007privacy,hardy2017private,cheng2019secureboost}.
We  discuss the data security (in other words,  prevent local data on one worker leaked to or inferred by  other workers) of FDSKL  under the \textit{semi-honest} assumption. Note that  the  \textit{semi-honest} assumption is commonly used in in security analysis \citep{wan2007privacy,hardy2017private,cheng2019secureboost}.
\begin{assumption}[Semi-honest security] \label{ass_semi_honest}
All workers  will follow the protocol or algorithm to perform the
correct computations. However, they may retain records of
the intermediate computation results which they may use
later to infer the data of other workers.
\end{assumption}

%\noindent \textbf{Data  Privacy:} \
Because each worker  knows the parameter $\omega$ given a  random seed, we can have a linear system of $o_j= (\omega_j)_{\mathcal{G}_{\ell}}^T (x_i)_{\mathcal{G}_{\ell}}$ with a sequence of trials of  $\omega_j$ and  $o_j$. It has the potential to  infer $(x_i)_{\mathcal{G}_{\ell}}$ from the linear system of $o_j=(\omega_j)_{\mathcal{G}_{\ell}}^T (x_i)_{\mathcal{G}_{\ell}}$ if the sequence of $o_j$ is also known\footnote{$o_j$ could be between an interval according to Algorithm \ref{protocol1}, we can guarantee the privacy of sample $x_i$ by enforcing the value of each element of  $x_i$ into a small range. Thus, we use normalized data in our algorithm.}. We call it inference attack. Please see its formal definition in Definition \ref{defi_inf_attack}.
  In this part, we will prove that  FDSKL can prevent  \textit{inference attack} (\textit{i.e.}, Theorem \ref{thm_privacy}). %(\textit{i.e.}, Theorem \ref{theorem_attack}).
\begin{definition}[Inference attack]\label{defi_inf_attack}
An inference attack on the $\ell$-th worker is to infer a certain feature group $\mathcal{G}$ of sample $x_i$ which belongs  to other workers
 without directly accessing it.
%based on a linear system of $(\omega)_{\mathcal{G}}^T (x_i)_{\mathcal{G}}=o$.
\end{definition}
\begin{theorem}\label{thm_privacy}
Under the \textit{semi-honest} assumption,  the FDSKL algorithm can prevent  \textit{inference attack}.
\end{theorem}
As discussed above, the key of preventing the inference attack is to  mask the value of $o_j$. As described in lines \ref{step2}-\ref{step3} of Algorithm \ref{protocol1}, we add an extra random variable  $b^{\hat{\ell}}$  into $(\omega_j)_{\mathcal{G}_{\ell}}^T (x_i)_{\mathcal{G}_{\ell}}$. Each time, the algorithm only transfers the value of $(\omega_i)_{\mathcal{G}_\ell}^T (x_i)_{\mathcal{G}_\ell}+b^{\hat{\ell}} $ to another worker. Thus, it is impossible for the receiver worker to directly infer the value of $o_j$. Finally,  the $\ell$-th active worker gets  the global sum $\sum_{\hat{\ell}=1}^q  \left ((\omega_i)_{\mathcal{G}_{\hat{\ell}}}^T (x_i)_{\mathcal{G}_{\hat{\ell}}}+b^{\hat{\ell}} \right )$ by using a tree-structured communication scheme based on the tree structure $T_1$. Thus, lines \ref{step2}-\ref{step4} of Algorithm \ref{protocol1} keeps data privacy.

As proved in Lemma \ref{lemma1}, lines \ref{step5}-\ref{step7} of Algorithm \ref{protocol1} is to get $\omega_i^T x+b$ by  removing $\overline{b}^{\ell'}=\sum_{{{\hat{\ell}}} \neq \ell'} b^{{{\hat{\ell}}}}$ from the sum $\sum_{\hat{\ell}=1}^q  \left ((\omega_i)_{\mathcal{G}_{\hat{\ell}}}^T (x_i)_{\mathcal{G}_{\hat{\ell}}}+b^{\hat{\ell}} \right )$. To prove that  FDSKL can prevent the inference attack, we only need to prove that the calculation of $\overline{b}^{\ell'}=\sum_{{{\hat{\ell}}} \neq \ell'} b^{{{\hat{\ell}}}}$ in line \ref{step7} of Algorithm \ref{protocol1} does not disclose the value of $b^{{{\hat{\ell}}}}$ or the sum of $b^{{{\hat{\ell}}}}$ on a node of tree $T_1$, which is  indicated in Lemma \ref{lemma2} (the proof is provided in Appendix).

% Because To keep the vertically partitioned  data privacy, we use a random seed received from one worker to  generate the random feature $\omega_i$, and locally compute $(\omega_i)_{\mathcal{G}_\ell}^T (x_i)_{\mathcal{G}_\ell} $ which  avoids directly transferring the local data $(x_i)_{\mathcal{G}_\ell}$ to other workers. To prevent inferring the information of $(x_i)_{\mathcal{G}_\ell}$ by other worker, we   generate $b^\ell$ , and
%
%. Finally, we use a tree-structured communication scheme  to obtain $\xi=\sum_{\ell'=1}^q (\omega_i)_{\mathcal{G}_{\ell'}}^T (x_i)_{\mathcal{G}_{\ell'}}+b^{\ell'}$.
%In order to prevent the inference attack, we must prevent to infer $b^{\hat{\ell}}$ which is related to $o_j$.

\begin{lemma}\label{lemma2}
In Algorithm \ref{protocol1}, if $T_1$ and $T_2$ are totally different tree structures, for any worker ${\hat{\ell}}$, there is no risk to disclose the value of $b^{{{\hat{\ell}}}}$ or the sum of $b^{{{\hat{\ell}}}}$ to other workers.
% Using a  tree structure $T_2$ for workers  $\{1, \ldots,q\}-\{\ell'  \}$  which is totally different to  the tree $T_1$ to compute $\overline{b}^{\ell'}=\sum_{{\hat{\ell}} \neq \ell'} b^{{\hat{\ell}}}$, for any worker, there is no risk to  disclose the value of $b^{{{\hat{\ell}}}}$ or the sum of $b^{{{\hat{\ell}}}}$ on other workers.
\end{lemma}
% \begin{theorem}\label{theorem_attack}
%Our  FDSKL can
%prevent the inference attack as defined in Definition \ref{defi_inf_attack}.
%\end{theorem}

%
%Note that if  the value of $b^\ell$ or  the sum of $b^\ell$ on a subtree of $T_1$ is revealed, the information of $(x_i)_{\mathcal{G}_\ell}$ or $x_i$ on the subtree of $T_1$  has the potential to be leaked. In order to prevent leaking any information of $b^\ell$ and  the sum of $b^\ell$ on a subtree of $T_1$,
%
%
%
% To recover $\omega_i^T x+b$,  we use a different tree-structured communication scheme to compute $\overline{b}$ (the sum of  $q-1$ values of $b^\ell$). Note that the privacy of random variable $b^\ell $ plays a fundamental role to the privacy of $(x_i)_{\mathcal{G}_\ell}$. We use two significant  tree-structured communication schemes (please see Figure \ref{FigureTSC}) to avoid inferring the value of random variable $b^\ell $.
%

\subsection{Complexity Analysis}
The computational complexity for one iteration of FDSKL is $\mathcal{O}(dqt)$. The total computational complexity of FDSKL is $\mathcal{O}(dqt^2)$. Further, the    communication cost for one iteration of FDSKL is $\mathcal{O}(qt)$, and the total communication cost  of FDSKL is $\mathcal{O}(qt^2)$.
The details of deriving the computational complexity and communication cost of FDSKL are provided in Appendix.

%Thus, the  computational complexity of Algorithm \ref{alg:predict} is $\mathcal{O}(dq |\Lambda^\ell|)$ and the communication  cost of Algorithm \ref{alg:predict} is $\mathcal{O}(q |\Lambda^\ell|)$.
%\noindent \textbf{Communication Cost.} \ We derive the communication complexity of FDSKL as follows.
\section{Experiments}\label{sectionexperiments}
In this section, we first present the experimental setup, and then provide the experimental results and discussions.
\subsection{Experimental Setup}
\subsubsection{Design of Experiments}  To demonstrate the superiority of FDSKL on federated kernel learning with vertically partitioned data, we compare FDSKL with PP-SVMV \citep{yu2006privacy}, which is the state-of-the-art algorithm of the field.  %including  and feature-distributed SVRG (FD-SVRG) \citep{wan2007privacy}.
Additionally, we also compare with SecureBoost \citep{cheng2019secureboost}, which is recently proposed to generalize the gradient tree-boosting algorithm to federated scenarios.
Moreover, to verify the predictive accuracy of FDSKL on vertically partitioned data, we compare with oracle learners that can access the whole data samples without the federated learning constraint. For the oracle learners, we use state-of-the-art kernel classification solvers, including LIBSVM \citep{CC01a} and DSG \citep{dai2014scalable}. Finally, we include FD-SVRG \citep{wan2007privacy}, which uses a linear model, to comparatively verify the accuracy of FDSKL.
%and the linear method feature-distributed SVRG (FD-SVRG) algorithm \citep{wan2007privacy}.

	\begin{table}[htbp]
	\centering
	%\vspace*{-5pt}
%	\small
	\caption{The benchmark  datasets used in the experiments.}
	%\vspace*{-8pt}
	\setlength{\tabcolsep}{6mm}
	\begin{tabular}{c|c|c}
		\hline
		\textbf{Datasets} &   \textbf{Features} &  \textbf{Sample size}   \\
		\hline
		gisette & 5,000 & 6,000\\
	%	\hline
		phishing  & 68 & 11,055\\
	%	\hline
		a9a & 123 & 48,842\\
	%	\hline
		ijcnn1 & 22 & 49,990\\
	%	\hline
		cod-rna  & 8 & 59,535\\
	%	\hline
		w8a  & 300 & 64,700\\
	%	\hline
		real-sim  & 20,958 & 72,309\\
	%	\hline
		epsilon  & 2,000 & 400,000\\
		%\hline
		%Epsilon & 2 & 2,000 & 200,000 & 100 $\%$ & AUC \\
	    \hline
	    defaultcredit & 23 & 30,000\\
	    givemecredit & 10 & 150,000\\
		\hline
	\end{tabular}
	\label{datasets_table}
%	\vspace*{-3pt}
\end{table}

\subsubsection{Implementation Details}
Our experiments were performed on a 24-core two-socket Intel Xeon CPU E5-2650 v4 machine with 256GB RAM.  We implemented our FDSKL in python, where the parallel computation was handled via  MPI4py  \citep{dalcin2011parallel}. We utilized the SecureBoost algorithm through the official unified framework \footnote{The code is available at \url{https://github.com/FederatedAI/FATE}.}.  The code of  LIBSVM is provided by \cite{CC01a}.   We used the   implementation\footnote{The DSG code is available at \url{https://github.com/zixu1986/Doubly_Stochastic_Gradients}.}  provided by  \cite{dai2014scalable} for  DSG.  We modified the implementation of DSG such that it uses constant learning rate. Our experiments use the following binary classification datasets as described below.
%In the experiments, we only consider binary classification problem.
%and regression problems. Specifically, our FDSKL uses the logistic loss (see Table \ref{table:IP}) for  binary classification, and  the square loss (see Table \ref{table:IP}) for  regression.
%We use the accuracy as the measure criterion of binary classification,  and use the  mean squared error (MSE) as the measure criterion of regression. In each experiment, the accuracy and MSE  values  are the average over 5 trials. In the experiments, the value of steplength $\gamma$ is selected from $\{10^2; 10^1; 10^{-1}; 10^{-2}; 10^{-3}; 10^{-4}; 10^{-5}\}$.

\subsubsection{Datasets}  Table \ref{datasets_table} summarizes the eight  benchmark binary classification datasets and two real-world financial datasets  used  in our experiments. The first eight benchmark datasets are obtained from LIBSVM website \footnote{These datasets  are from \url{https://www.csie.ntu.edu.tw/~cjlin/libsvmtools/datasets/}.}, the defaultcredit dataset is from the UCI  \footnote{\url{https://archive.ics.uci.edu/ml/datasets/default+of+credit+card+clients}.} website, and the givemecredit dataset is from the  Kaggle \footnote{\url{https://www.kaggle.com/c/GiveMeSomeCredit/data}.} website.
% obtained from UCI   and Kaggle  respectively
%Covtype$\_$B and Covtype$\_$M are from a same source. Covtype$\_$M, MNIST and Aloi are originally for multi-class classification. In the experiments, we treat the  multi-class classification as regression problem.
We split the dataset as $3:1$ for training and testing, respectively. Note that in the experiments of the \textit{w8a}, \textit{real-sim}, \textit{givemecredit} and \textit{epsilon} datasets,  PP-SVMV always runs out of memory, which means this method only works when the number of instance is below around 45,000 when using the computation resources specified above. %(8 nodes).
Since the training time is over 15 hours, the result of SecureBoost algorithm on \textit{epsilon} dataset is absent.
\begin{figure*}[!ht]
%\vspace*{-4pt}
%	\centering
	\captionsetup[subfigure]{aboveskip=2pt,belowskip=-2pt}
	\centering
	\begin{subfigure}[b]{0.24\textwidth}
		\centering
		\includegraphics[width=1.5in]{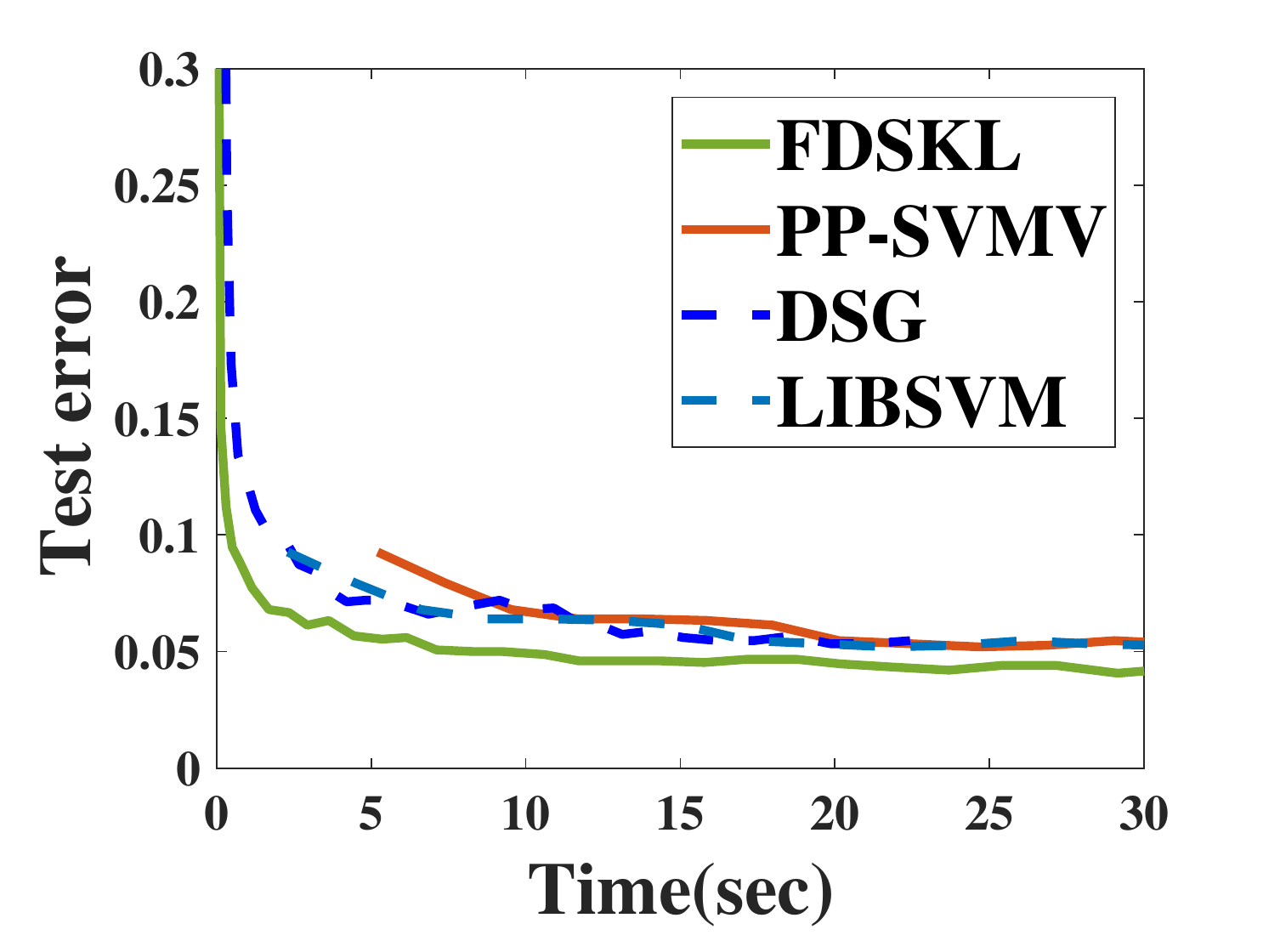}
		\label{gisette}
				%\vspace*{-13pt}
		\caption{gisette}
	\end{subfigure}
	\begin{subfigure}[b]{0.24\textwidth}
		\centering
		\includegraphics[width=1.5in]{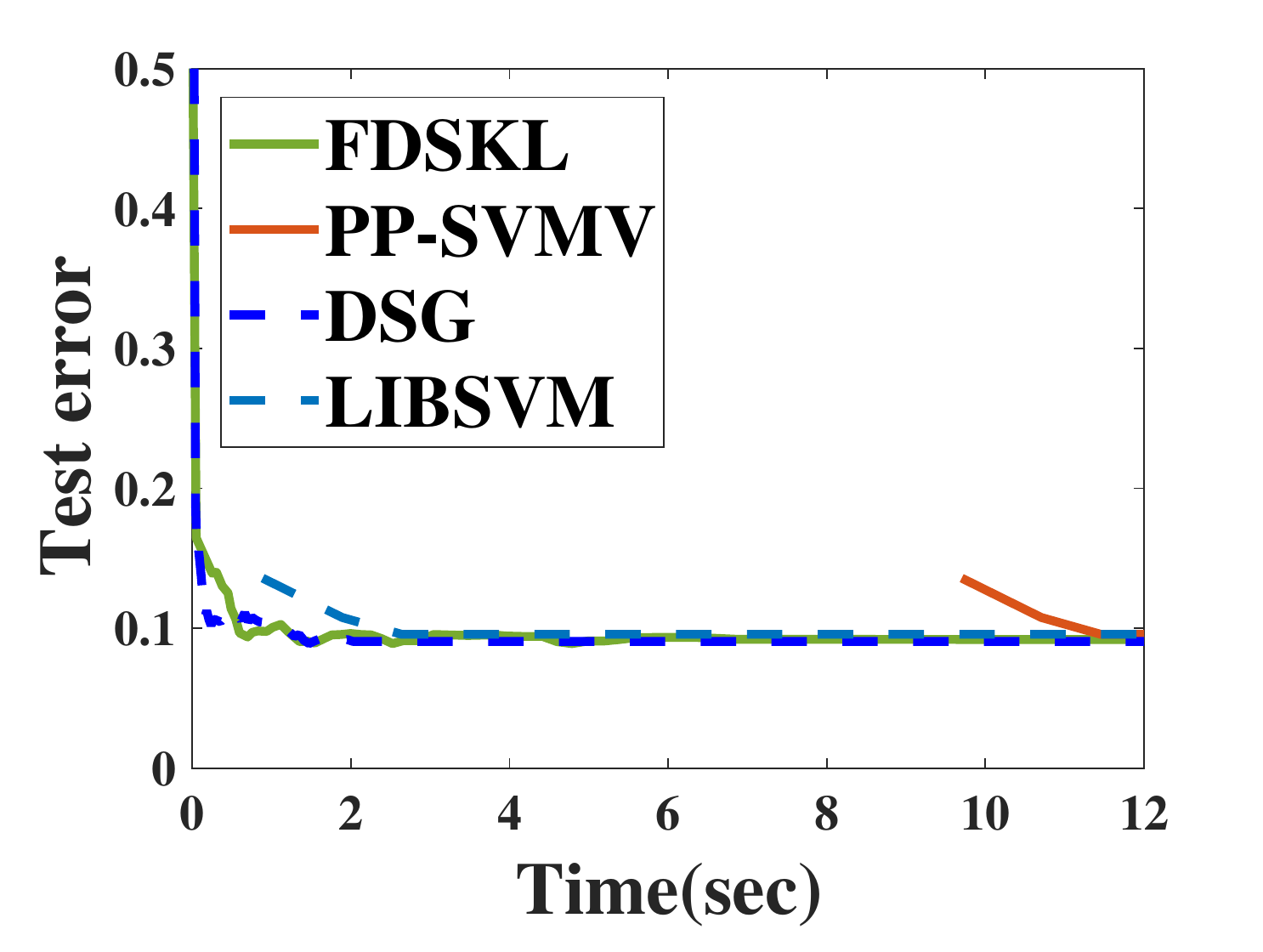}
		\label{phishing}
				%\vspace*{-13pt}
		\caption{phishing}
	\end{subfigure}
	\begin{subfigure}[b]{0.24\textwidth}
		\centering
		\includegraphics[width=1.5in]{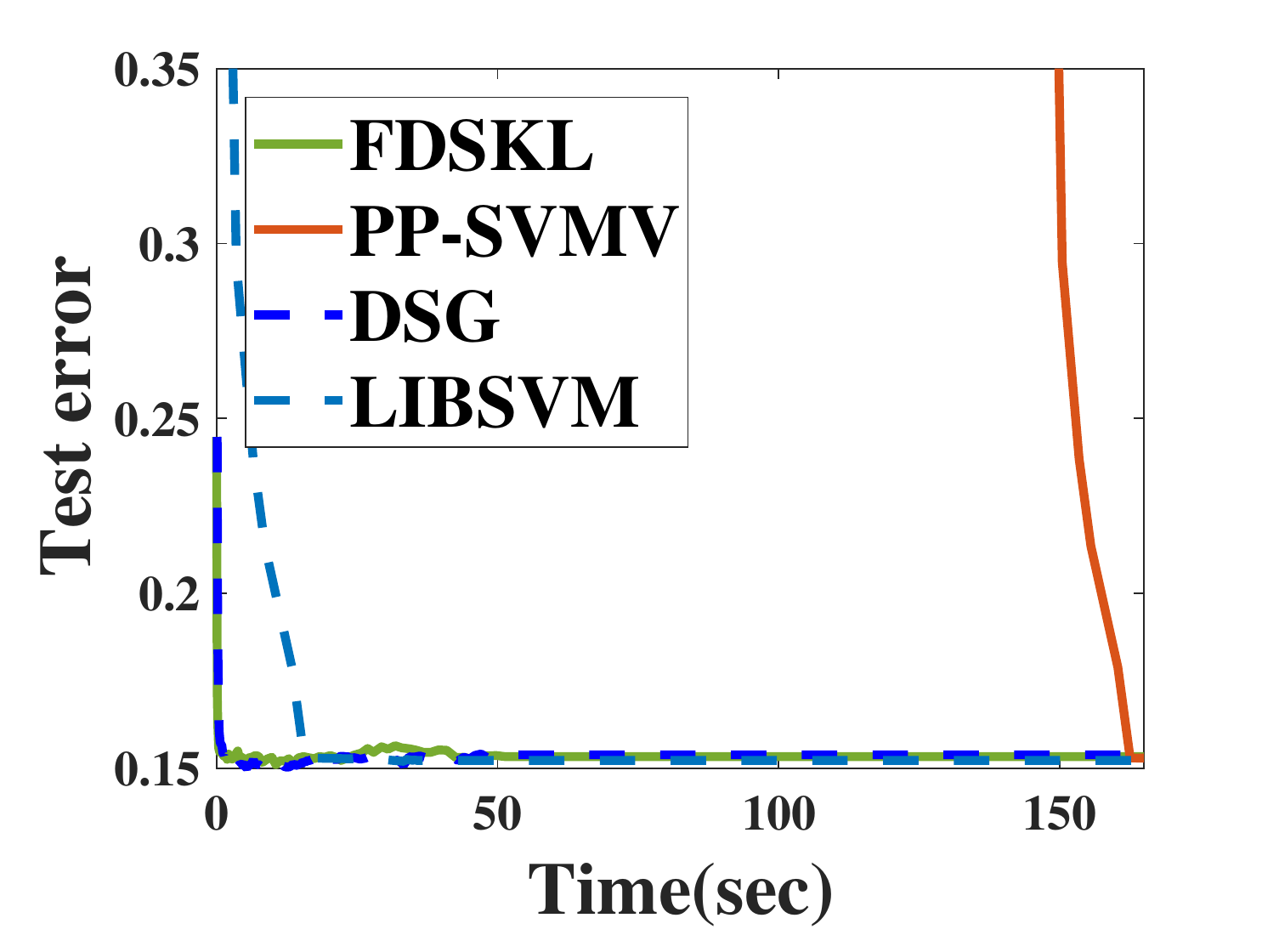}
		\label{a9a}
				%\vspace*{-13pt}
		\caption{a9a}
	\end{subfigure}
	\begin{subfigure}[b]{0.24\textwidth}
		\centering
		\includegraphics[width=1.5in]{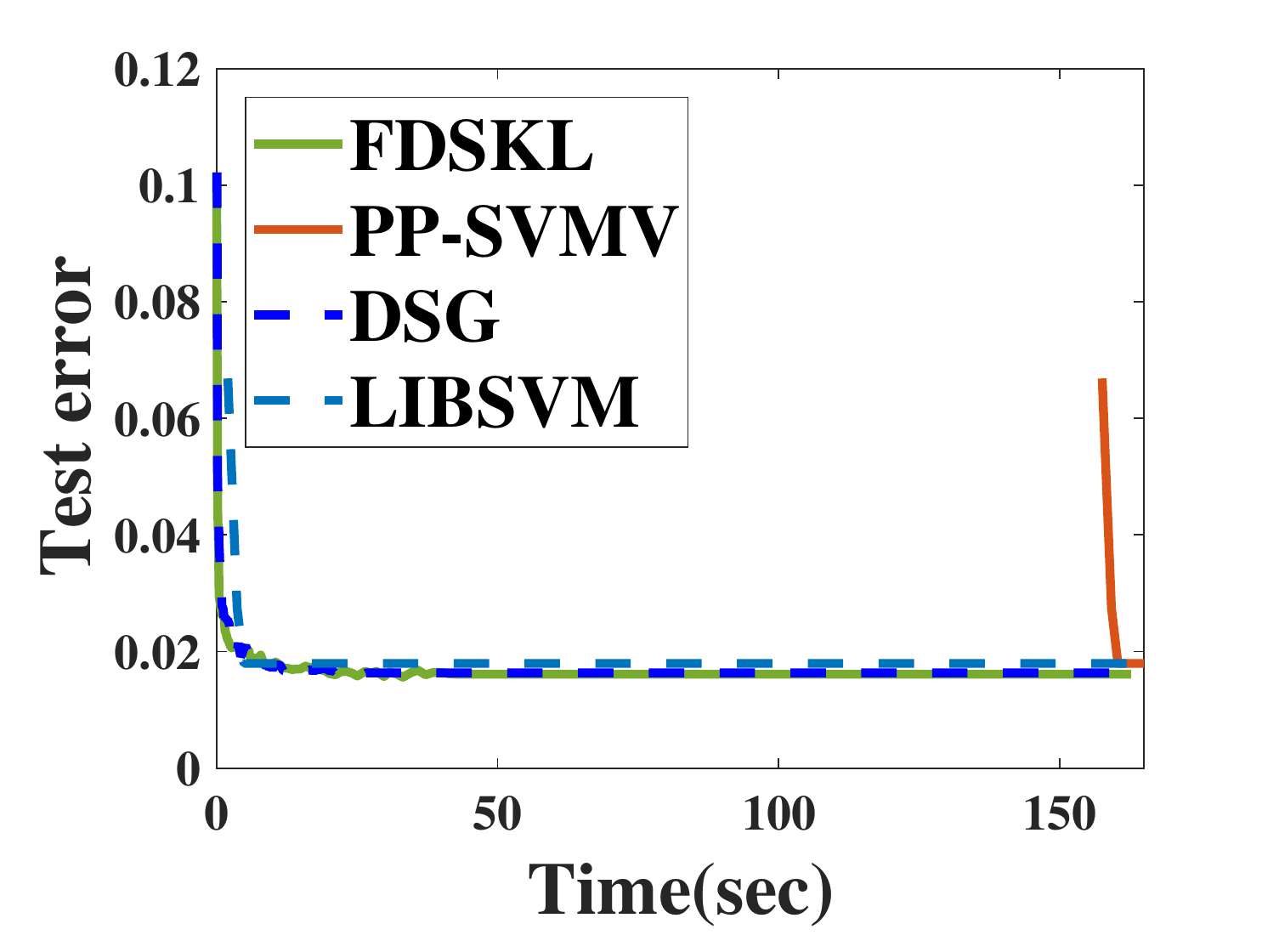}
		\label{ijcnn1}
				%\vspace*{-13pt}
		\caption{ijcnn1}
	\end{subfigure}
%	\\
	\begin{subfigure}[b]{0.24\textwidth}
		\centering
		\includegraphics[width=1.5in]{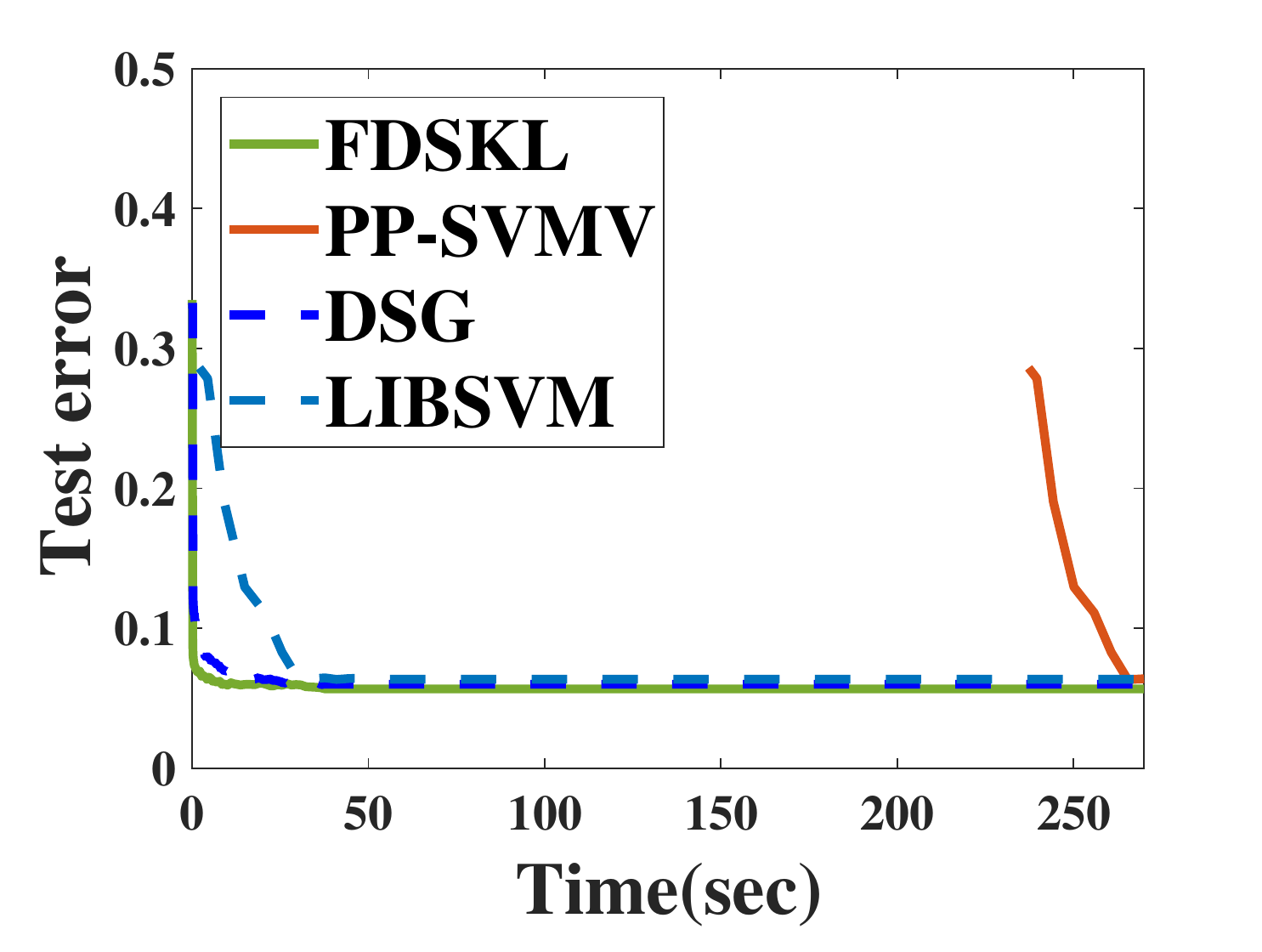}
		\label{cod}
				%\vspace*{-13pt}
		\caption{cod-rna}
	\end{subfigure}
	\begin{subfigure}[b]{0.24\textwidth}
		\centering
		\includegraphics[width=1.5in]{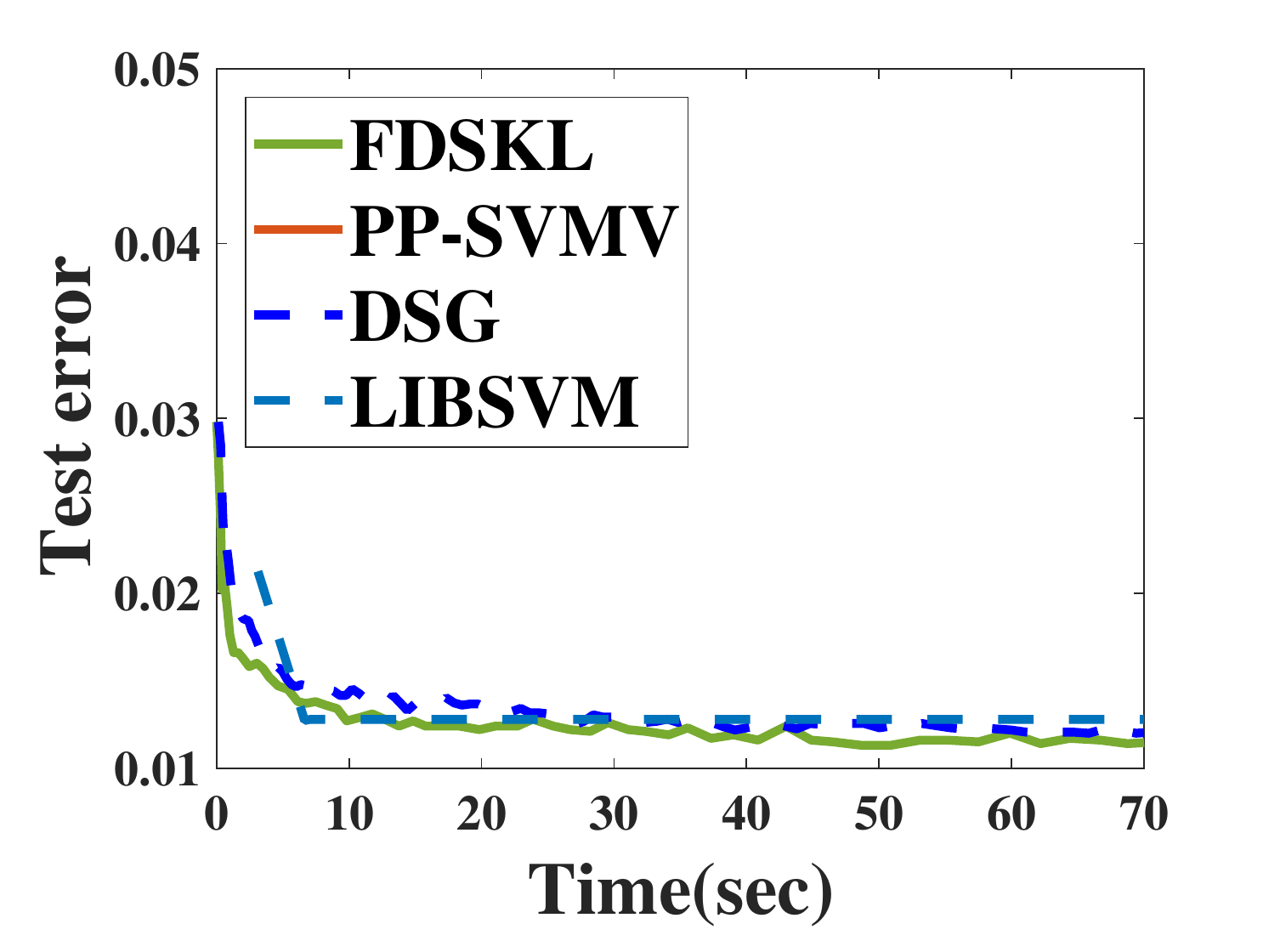}
		\label{w8a}
				%\vspace*{-13pt}
		\caption{w8a}
	\end{subfigure}
	\begin{subfigure}[b]{0.24\textwidth}
		\centering
		\includegraphics[width=1.5in]{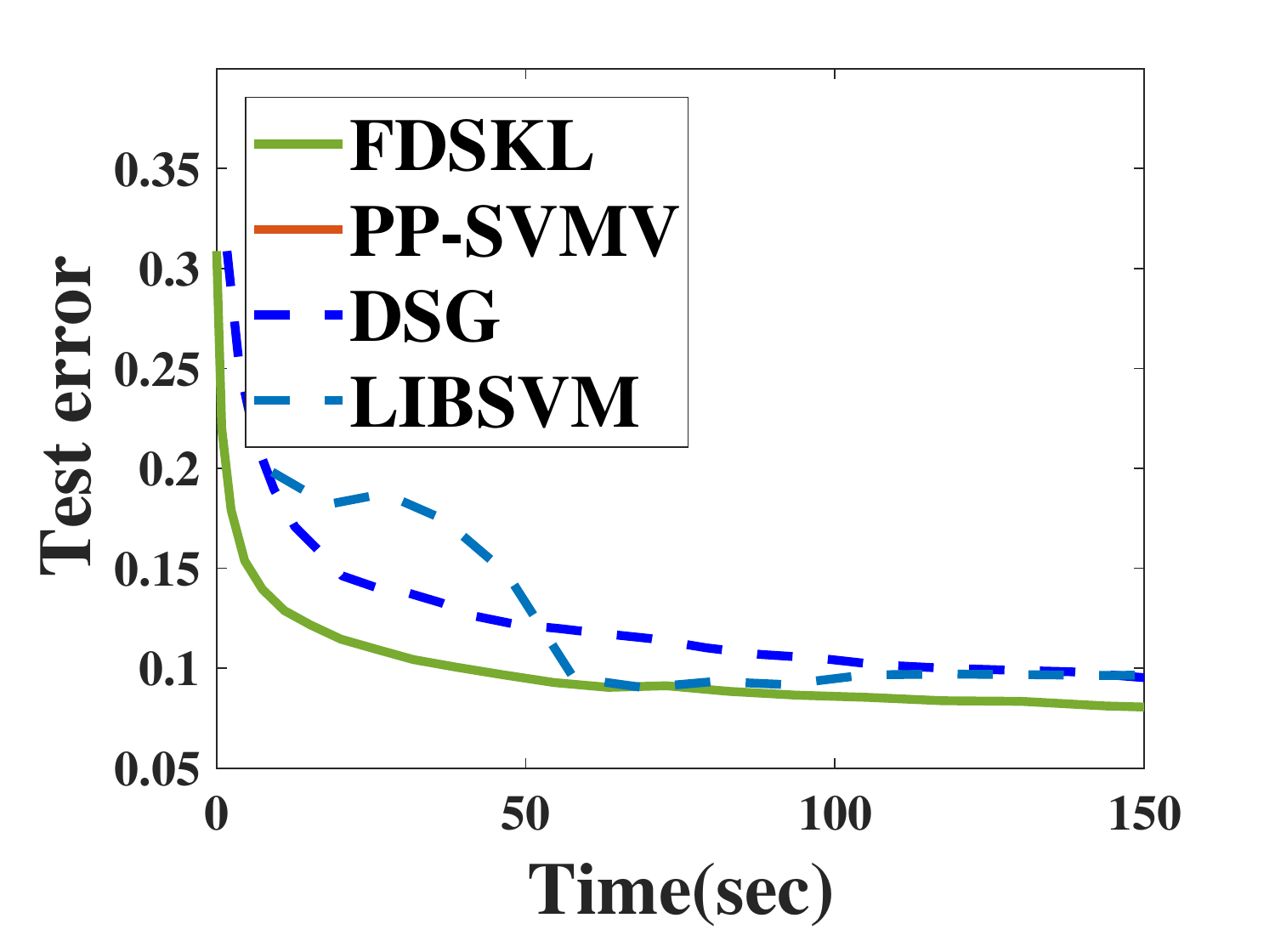}
		\label{real-sim}
				%\vspace*{-13pt}
		\caption{real-sim}
	\end{subfigure}
	\begin{subfigure}[b]{0.24\textwidth}
		\centering
		\includegraphics[width=1.5in]{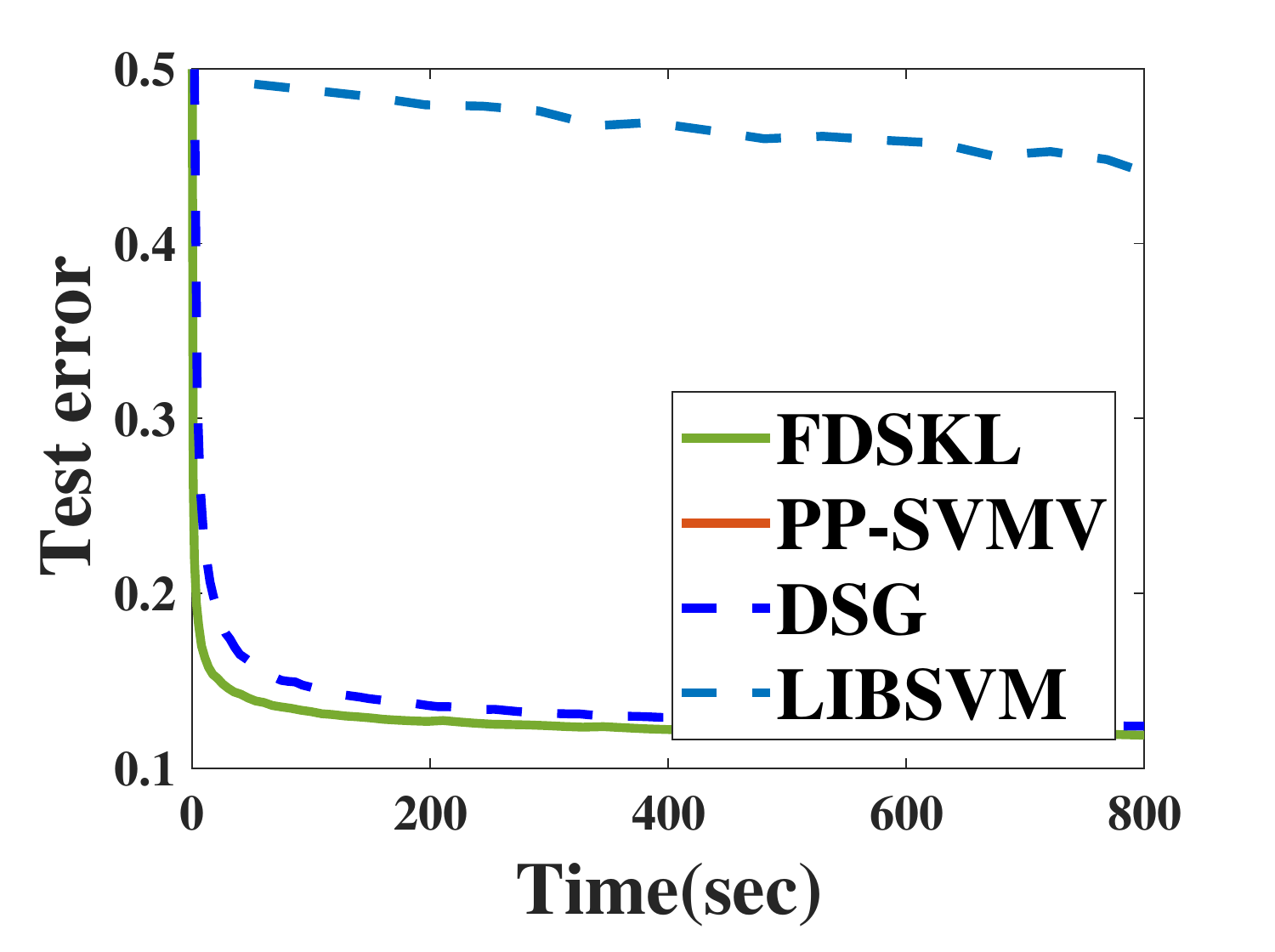}
		\label{epsilon}
				%\vspace*{-13pt}
		\caption{epsilon}
	\end{subfigure}
		\begin{subfigure}[b]{0.24\textwidth}
		\centering
		\includegraphics[width=1.5in]{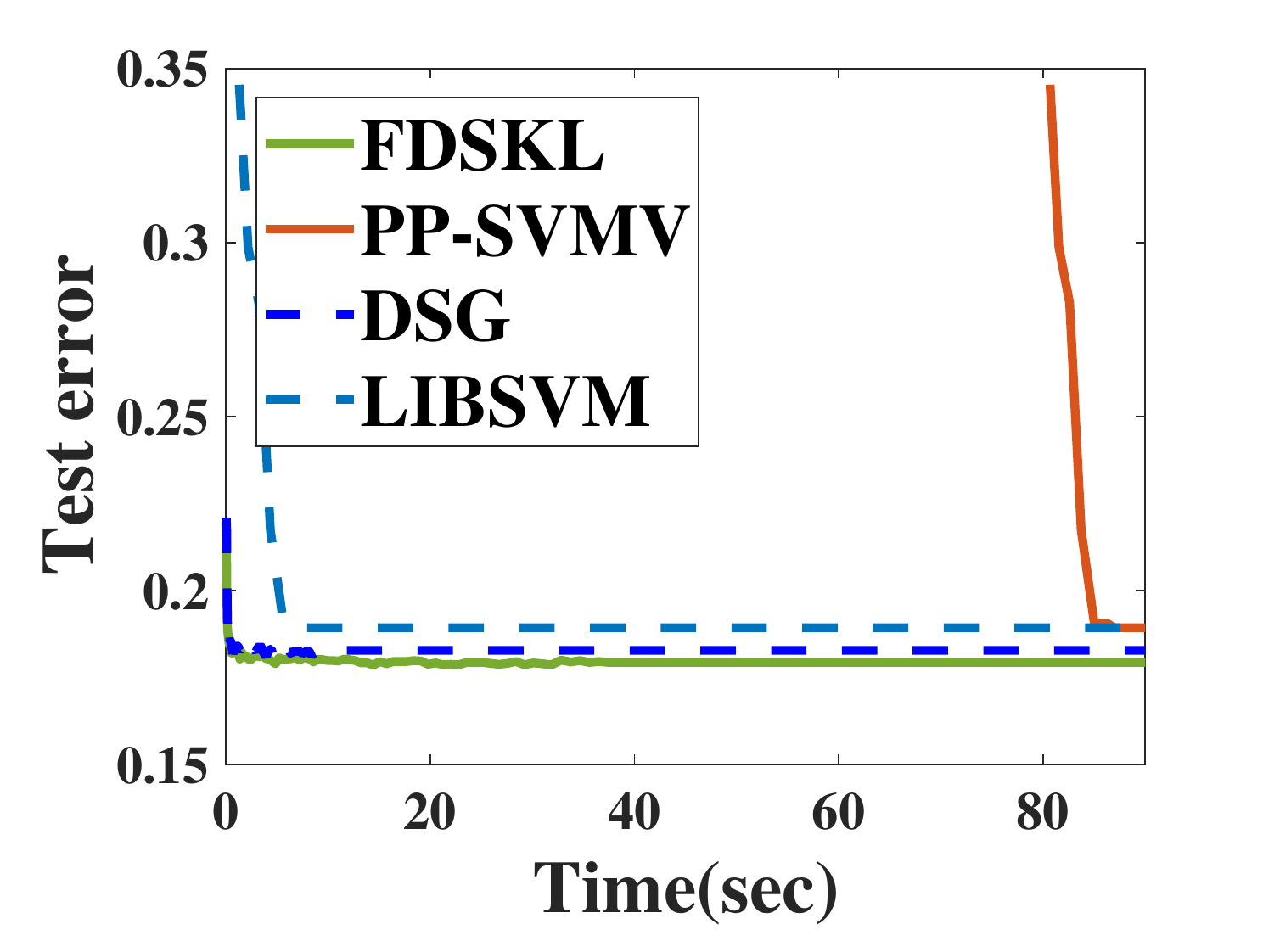}
		\label{deafultcredit}
%				%\vspace*{-13pt}
		\caption{deafultcredit}
	\end{subfigure}
	\begin{subfigure}[b]{0.24\textwidth}
		\centering
		\includegraphics[width=1.5in]{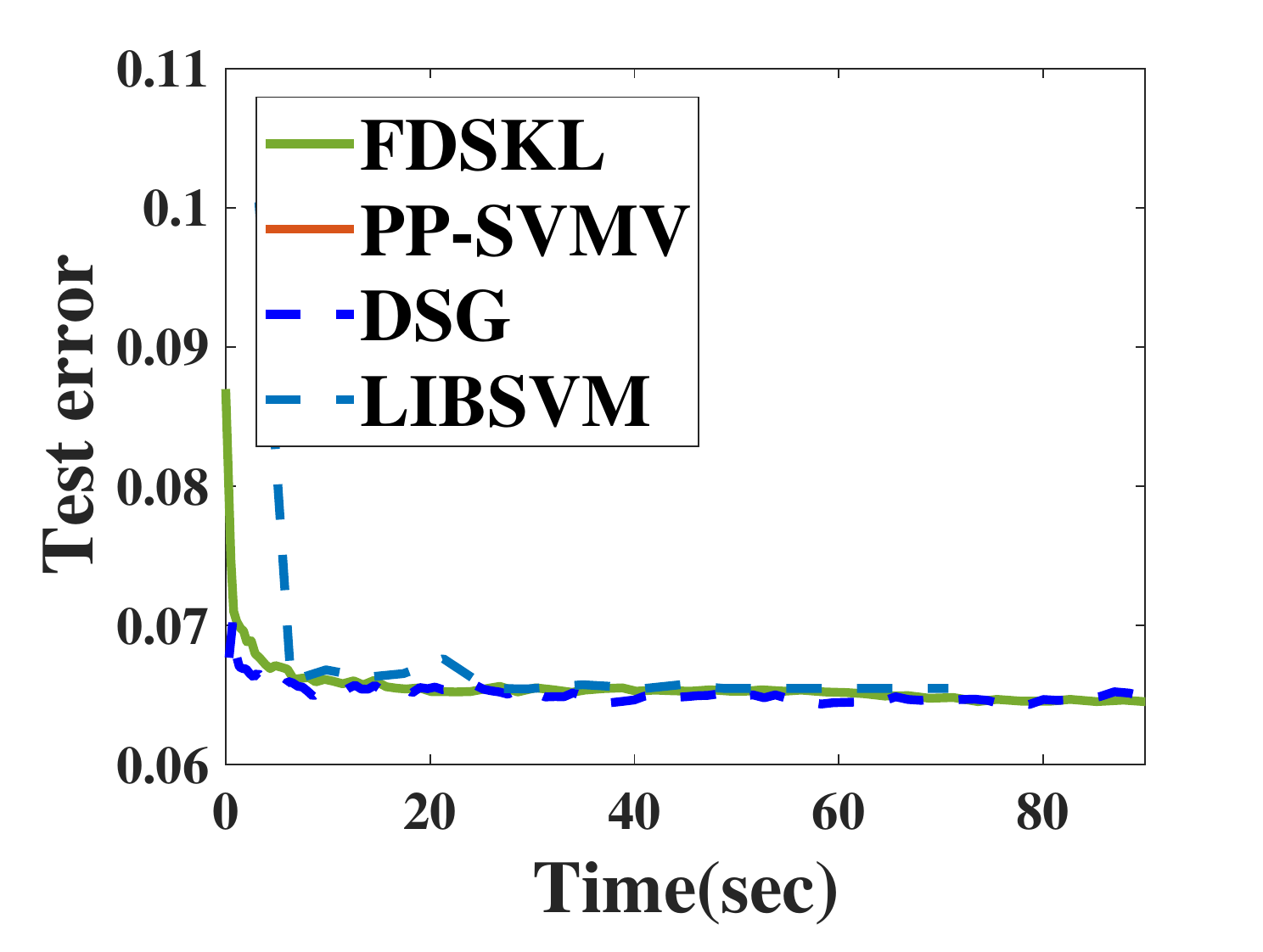}
		\label{givemecredit}
	%			\vspace*{-13pt}
		\caption{givemecredit}
	\end{subfigure}
	%\vspace*{-6pt}
	\caption{The results of binary classification above the comparison methods.}
	\label{results_bc}
%	\vspace*{-6pt}
\end{figure*}

\begin{figure*}[!ht]
%\vspace*{-4pt}
%	\captionsetup[subfigure]{aboveskip=1.5pt,belowskip=-1.5pt}
	\centering
	\begin{subfigure}[b]{0.24\textwidth}
		\centering
		\includegraphics[width=1.5in]{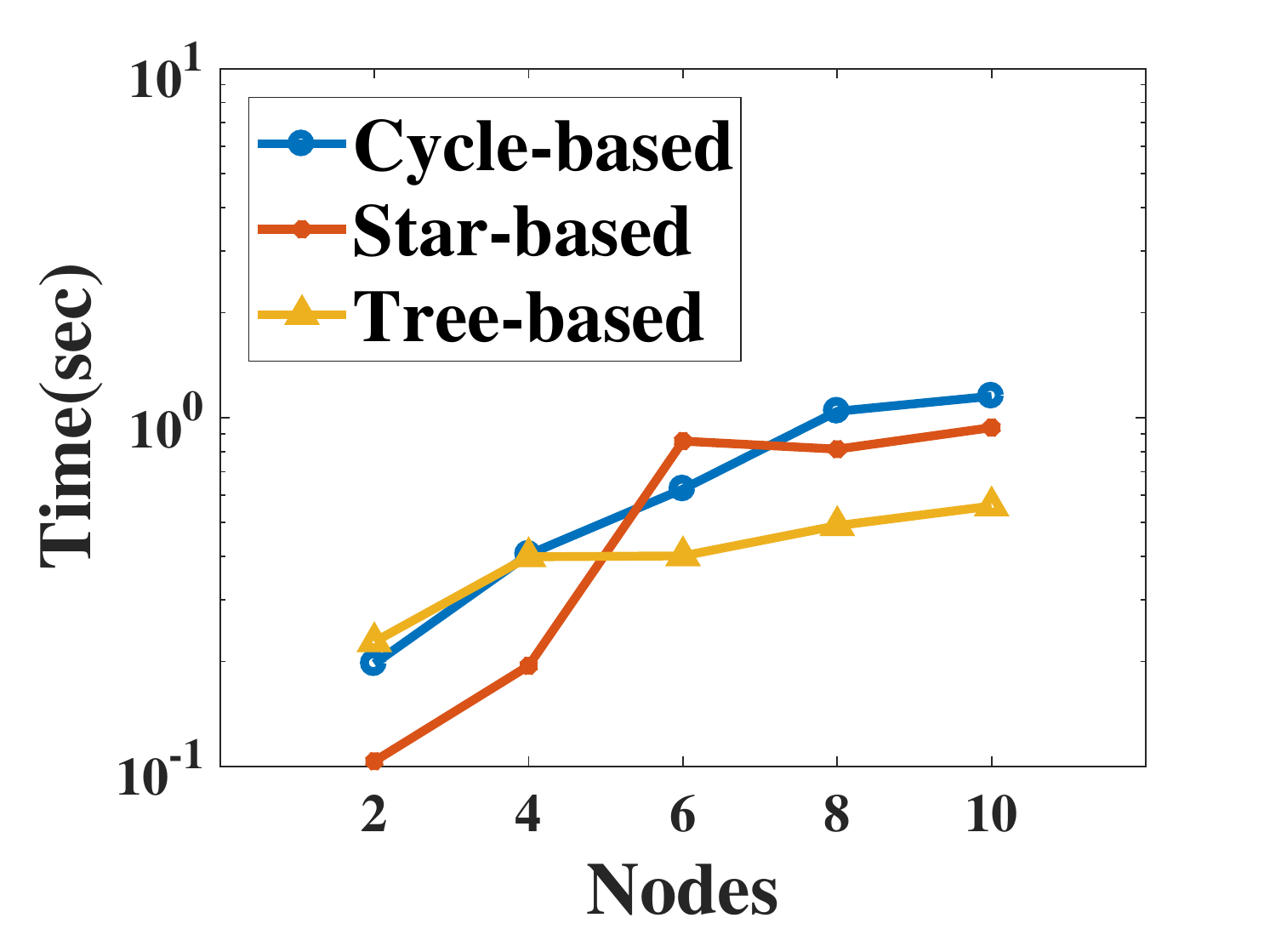}
		\label{com_gisette}
				%\vspace*{-13pt}
		\caption{gisette}
	\end{subfigure}
	\begin{subfigure}[b]{0.24\textwidth}
		\centering
		\includegraphics[width=1.5in]{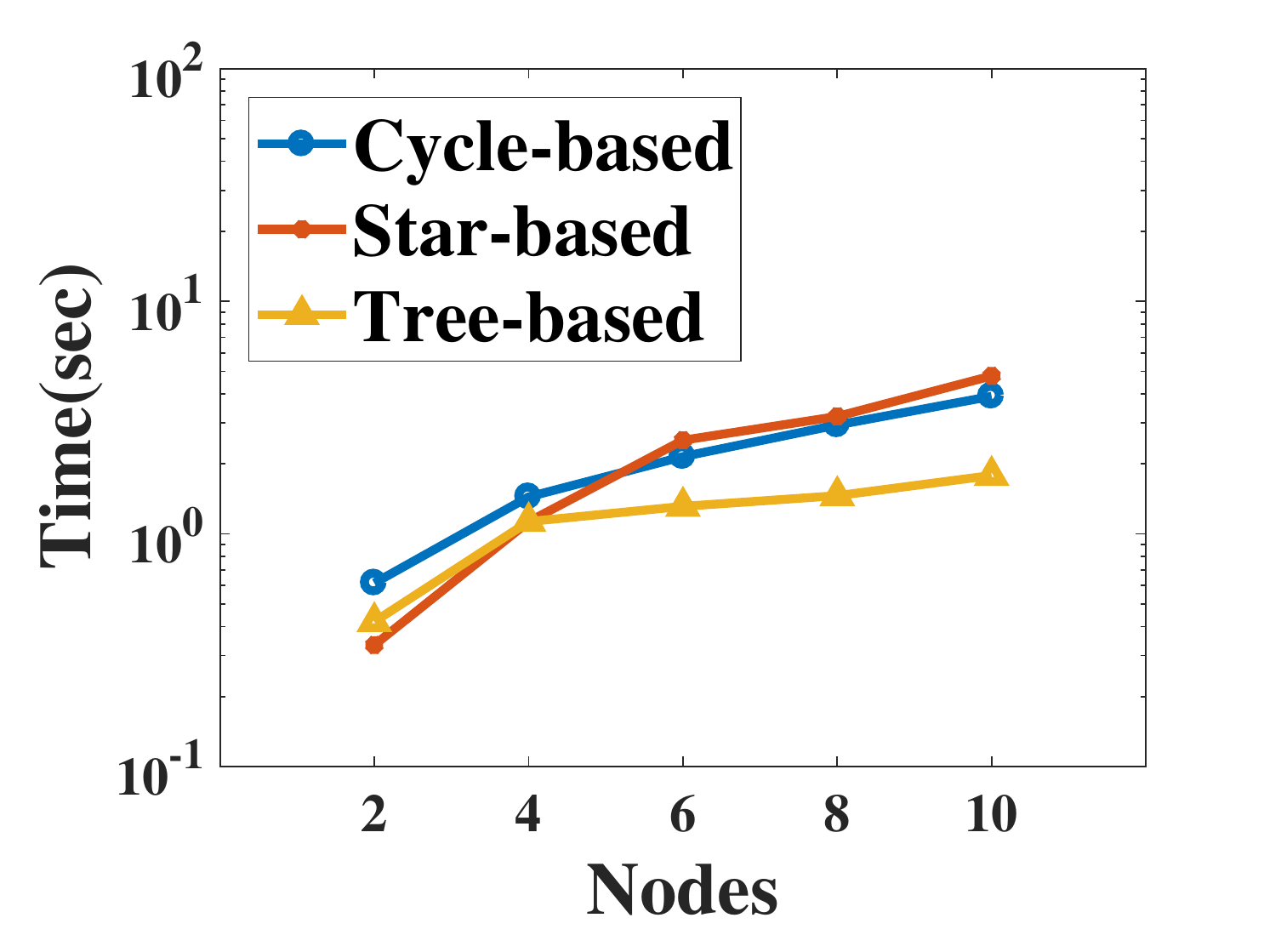}
		\label{com_phishing}
				%\vspace*{-13pt}
		\caption{phishing}
	\end{subfigure}
	\begin{subfigure}[b]{0.24\textwidth}
		\centering
		\includegraphics[width=1.5in]{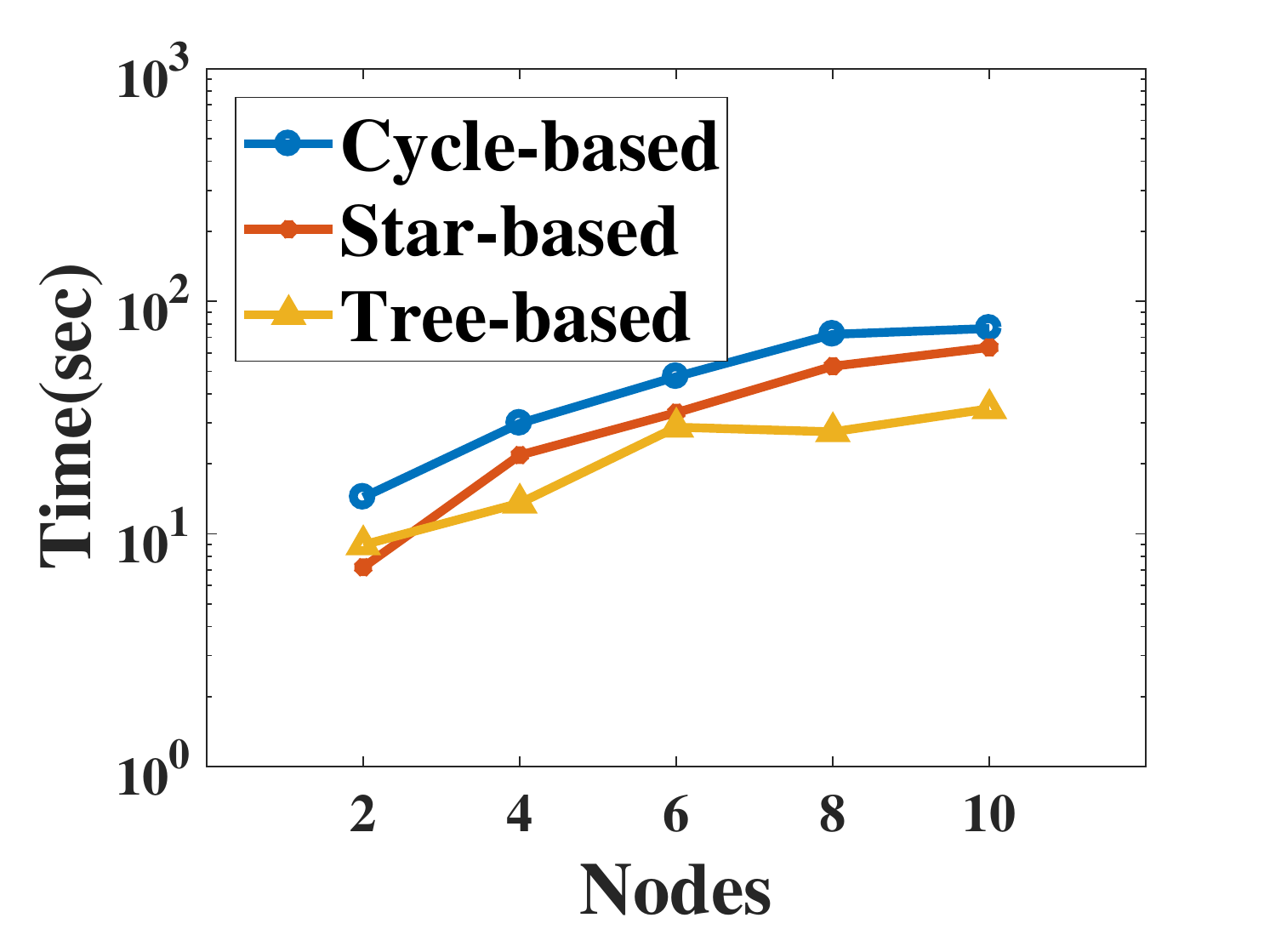}
		\label{com_a9a}
				%\vspace*{-13pt}
		\caption{a9a}
	\end{subfigure}
	\begin{subfigure}[b]{0.24\textwidth}
		\centering
		\includegraphics[width=1.5in]{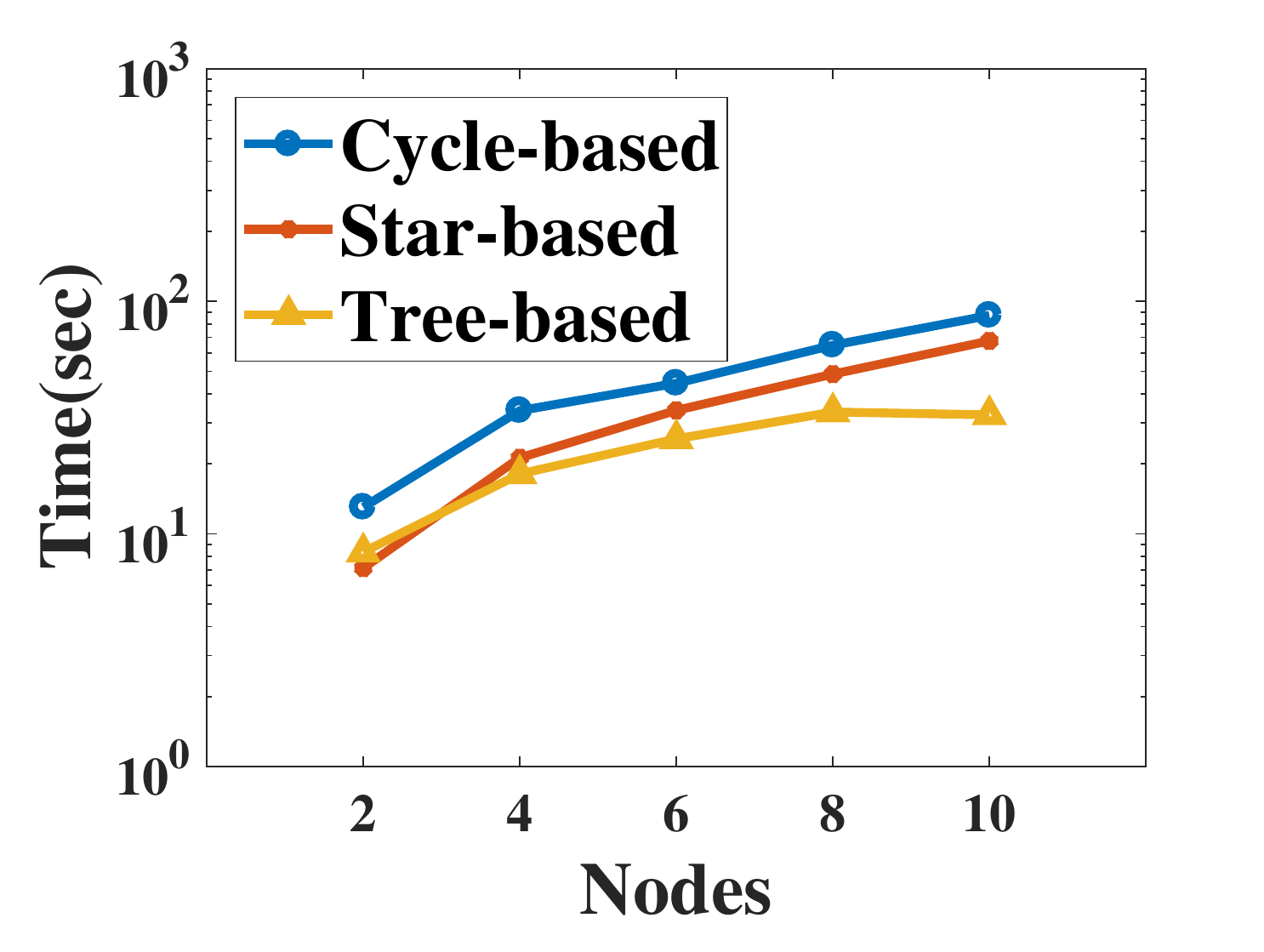}
		\label{com_ijcnn1}
				%\vspace*{-13pt}
		\caption{ijcnn1}
	\end{subfigure}
%	\vspace*{-6pt}
	\caption{The elapsed time of different structures on four datasets.}
	\label{results_com}
%	\vspace*{-6pt}
\end{figure*}

\begin{figure*}[!ht]
%\vspace*{-4pt}
%	\centering
%	\captionsetup[subfigure]{aboveskip=1.5pt,belowskip=-1.5pt}
	\centering
	\begin{subfigure}[b]{0.24\textwidth}
		\centering
		\includegraphics[width=1.5in]{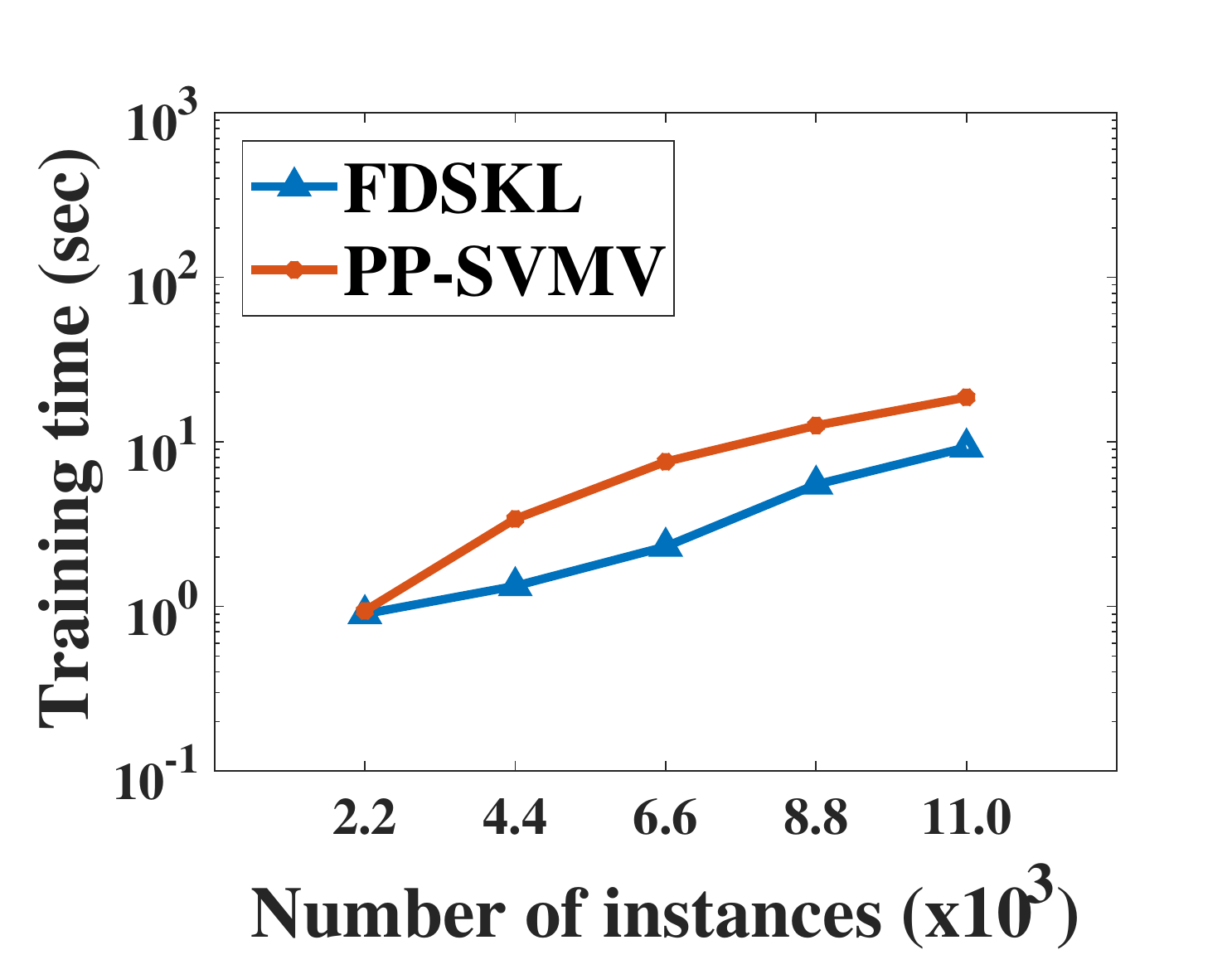}
		%		\vspace*{-13pt}
		\caption{phishing}
		\label{inc_phishing}
	\end{subfigure}
	\begin{subfigure}[b]{0.24\textwidth}
		\centering
		\includegraphics[width=1.5in]{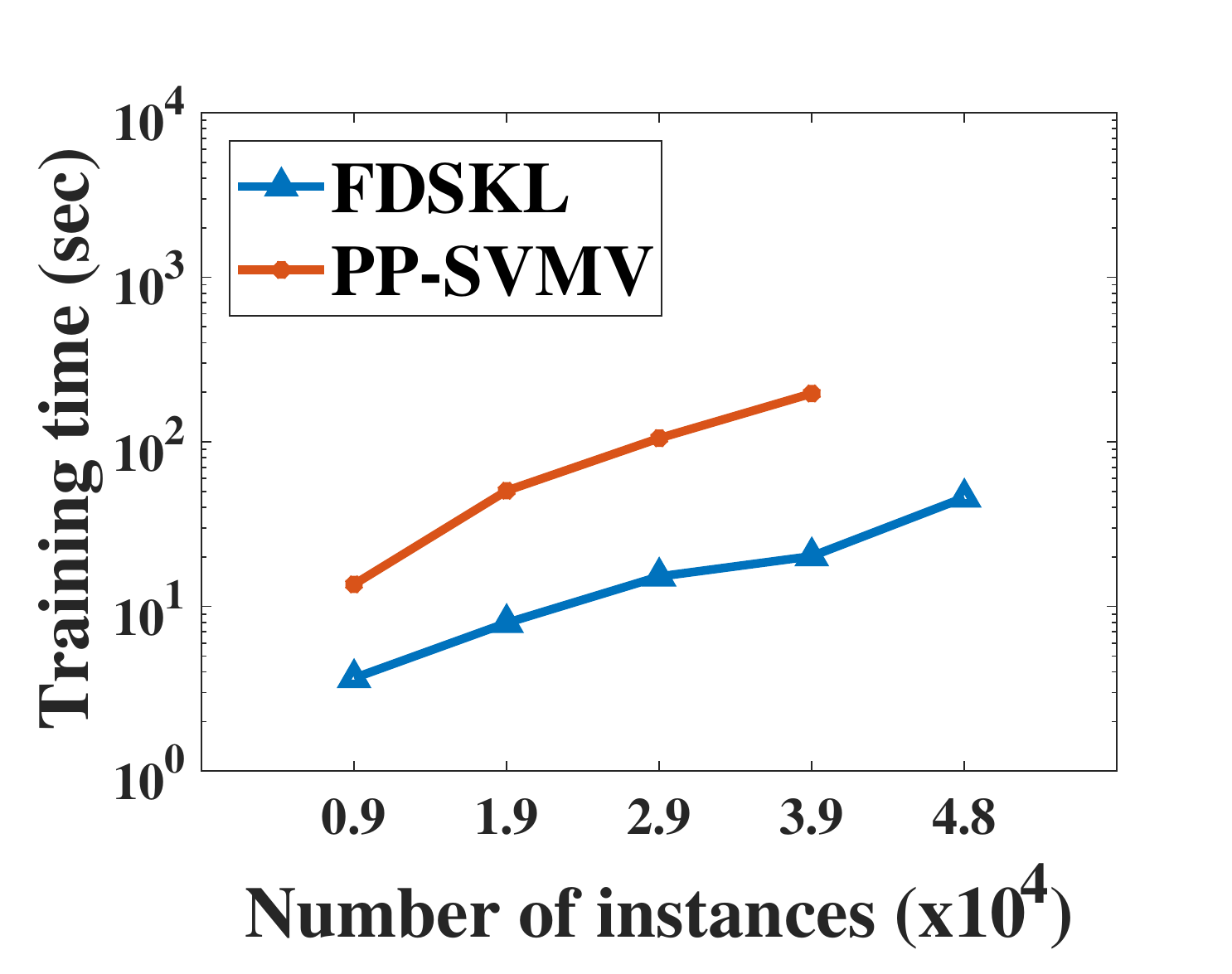}
		%		\vspace*{-13pt}
		\caption{a9a}
		\label{inc_a9a}
	\end{subfigure}
	\begin{subfigure}[b]{0.24\textwidth}
		\centering
		\includegraphics[width=1.5in]{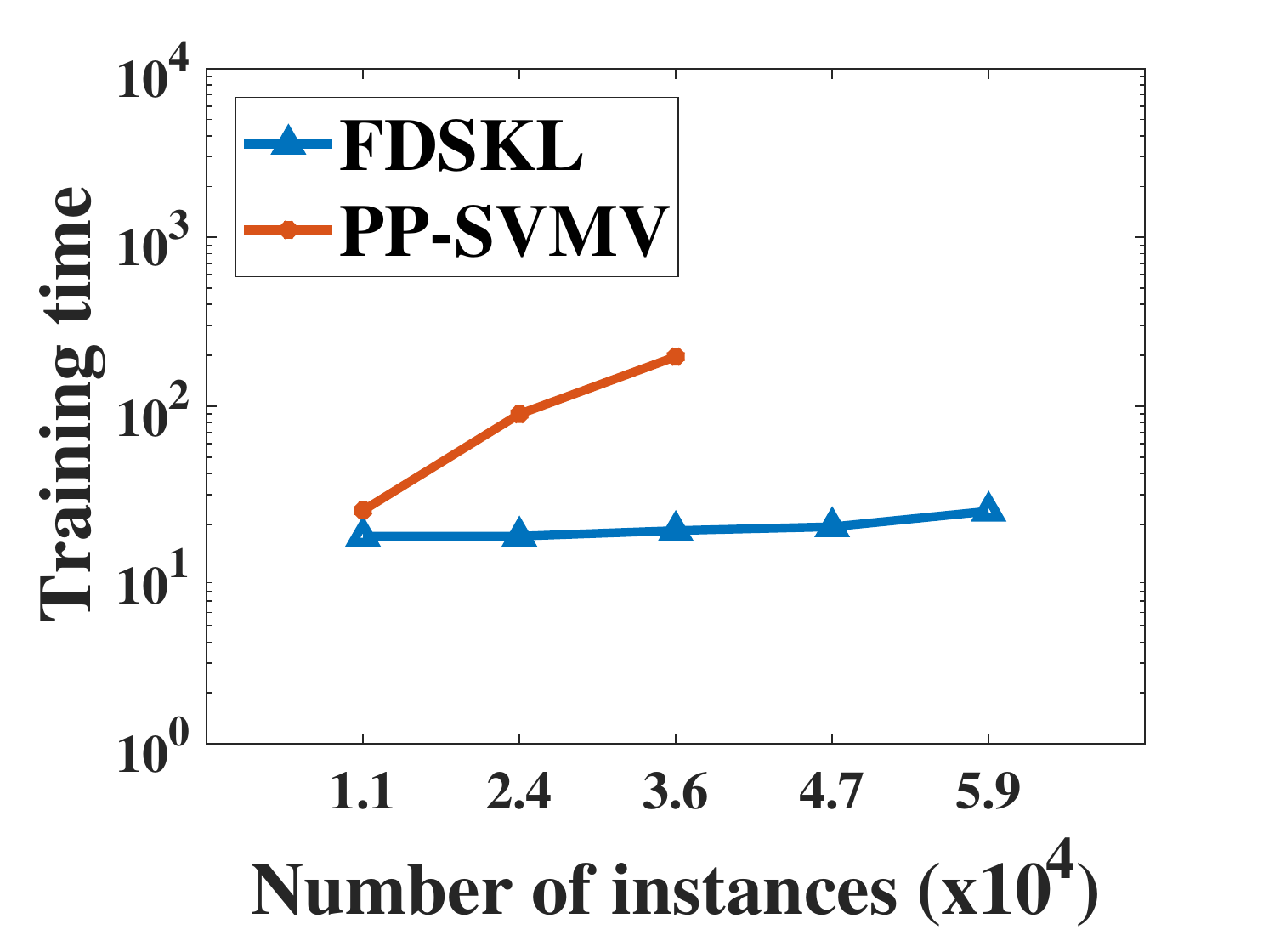}
		%		\vspace*{-13pt}
		\caption{cod-rna}
		\label{inc_cod}
	\end{subfigure}
	\begin{subfigure}[b]{0.24\textwidth}
		\centering
		\includegraphics[width=1.5in]{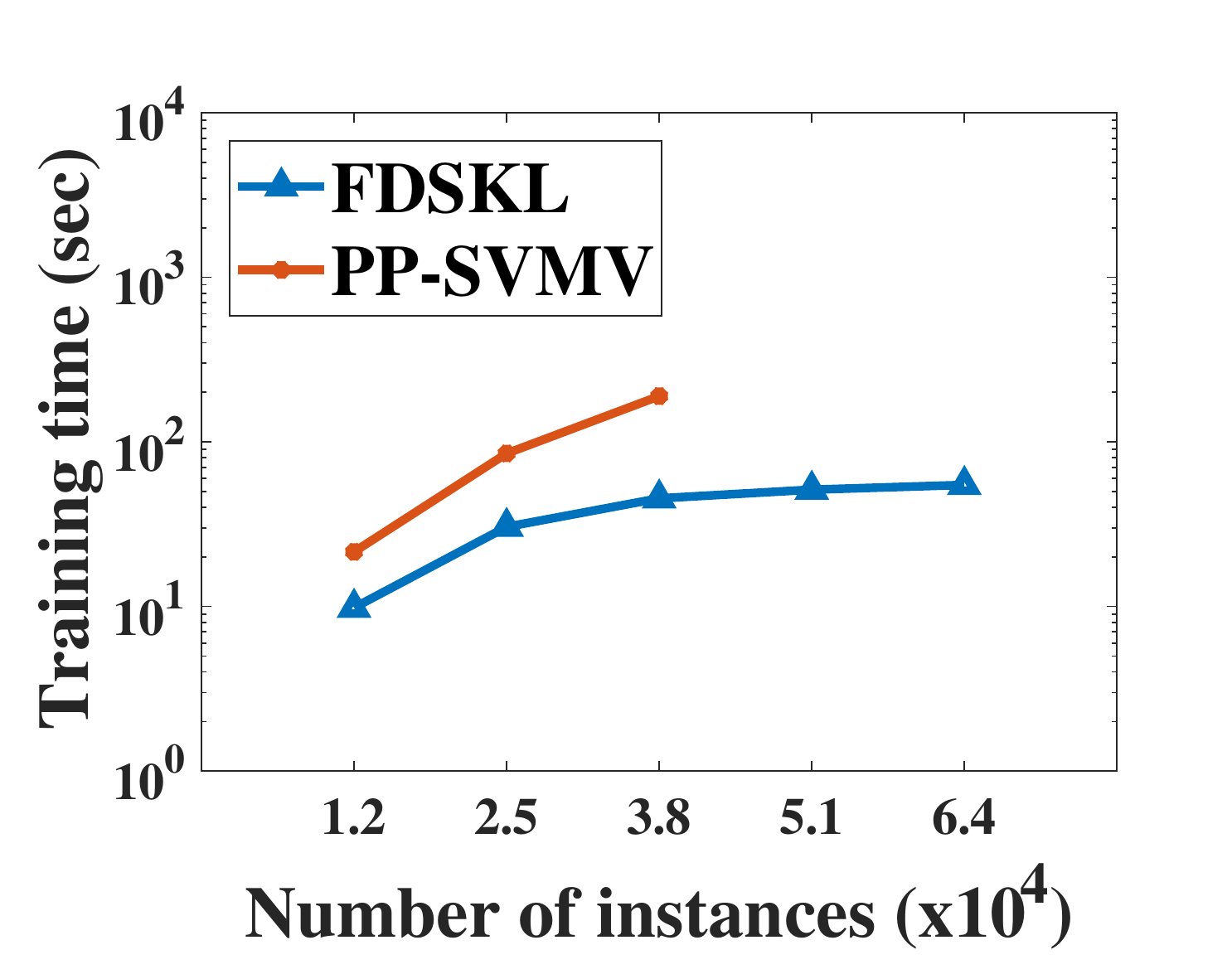}
		%		\vspace*{-13pt}
		\caption{w8a}
		\label{inc_w8a}
	\end{subfigure}
	%\vspace*{-6pt}
	\caption{The change of training time when increasing the number of training instances.}
	\label{results_inc}
	%\vspace*{-6pt}
\end{figure*}

%\begin{figure*}[!ht]
%	\centering
%	\begin{subfigure}[b]{0.24\textwidth}
%		\centering
%		\includegraphics[width=1.5in]{fig/increase_training/gisette}
%		\label{inc_gisette}
%		%		\vspace*{-13pt}
%		\caption{gisette}
%	\end{subfigure}
%	\begin{subfigure}[b]{0.24\textwidth}
%		\centering
%		\includegraphics[width=1.5in]{fig/increase_training/phishing}
%		\label{inc_phishing}
%		%		\vspace*{-13pt}
%		\caption{phishing}
%	\end{subfigure}
%	\begin{subfigure}[b]{0.24\textwidth}
%%		\centering
%		\includegraphics[width=1.5in]{fig/increase_training/a9a}
%		\label{inc_a9a}
%		%		\vspace*{-13pt}
%		\caption{a9a}
%	\end{subfigure}
%	\begin{subfigure}[b]{0.24\textwidth}
%		\centering
%		\includegraphics[width=1.5in]{fig/increase_training/cod}
%%		\label{inc_cod}
%		%		\vspace*{-13pt}
%		\caption{cod-rna}
%	\end{subfigure}
%	\vspace*{-5pt}
%	\caption{The training time when increasing the number of nodes}
%	\label{results_incnode}
%	\vspace*{-15pt}
%\end{figure*}

\subsection{Results and Discussions}
We provide the test errors \textit{v.s.} training time plot on four state-of-the-art kernel methods in Figure \ref{results_bc}. It is evident that our algorithm always achieves fastest convergence rate compared to other state-of-art kernel methods.
In Figure \ref{results_inc}, we demonstrate the training time v.s. different training sizes of FDSKL and PP-SVMV.
%We record the convergence time over the set which is (1/5,2/5,3/5,4/5,5/5) of the total dataset.
Again, the absence of the results in Figures \ref{inc_a9a}, \ref{inc_cod} and \ref{inc_w8a} for PP-SVMV is because of out of memory. It is obvious that our method has much better scalability than PP-SVMV. The reason for this edge of scalability is comprehensive, mostly because FDSKL have adopted the random feature approach, which is efficient and easily parallelizable. Besides, we could also demonstrate that the communication structure used in PP-SVMV is not optimal, which means more time spent in sending and receiving the partitioned kernel matrix.
%Similarly, because of the OOM of PP-SVMV, there are also some points missed in Figure (\ref{inc_a9a},\ref{inc_cod},\ref{inc_w8a}).

As mentioned in previous section, FDSKL used a tree-structured communication scheme to distribute and aggregate computation. To verify such a systematic design, we  compare the efficiency of 3 commonly used communication structures, \textit{i.e.}, cycle-based, tree-based and star-based communication structures. The goal of the comparison task is to compute the kernel matrix (linear kernel) of the training set of four datasets. Specifically, each node maintains a feature subset of the training set, and is asked to compute the kernel matrix using the feature subset only. The computed local kernel matrices on each node are then summed by using one of the three communication structures. Our experiment compares the efficiency (elapsed communication time) of obtaining the final kernel matrix, and the results are given in Figure \ref{results_com}.
%In this manner, we will obtain the total kernel matrix and for efficiency comparison, we record the elapsed communication time.
From  Figure \ref{results_com},  we could make a statement that with the increase of nodes, our tree-based communication structure has the lowest communication cost.
This explains the poor efficiency of PP-SVMV, which used a cycle-based communication structure, as given in Figure \ref{results_inc}.

%Considering the influence of nodes number, we compare the training time under different nodes with PP-SVMV in Fig.\ref{results_incnode}. According to the results,  we found that with the increase of node number, the training time of our algorithm barely changed, but PP-SVMV acts the opposite.
\begin{figure*}[!ht]
%\vspace*{-6pt}
%	\centering
%	\captionsetup[subfigure]{aboveskip=1.5pt,belowskip=-1.5pt}
	\centering
	\hspace*{-0.25cm}
	\includegraphics[height=2in]{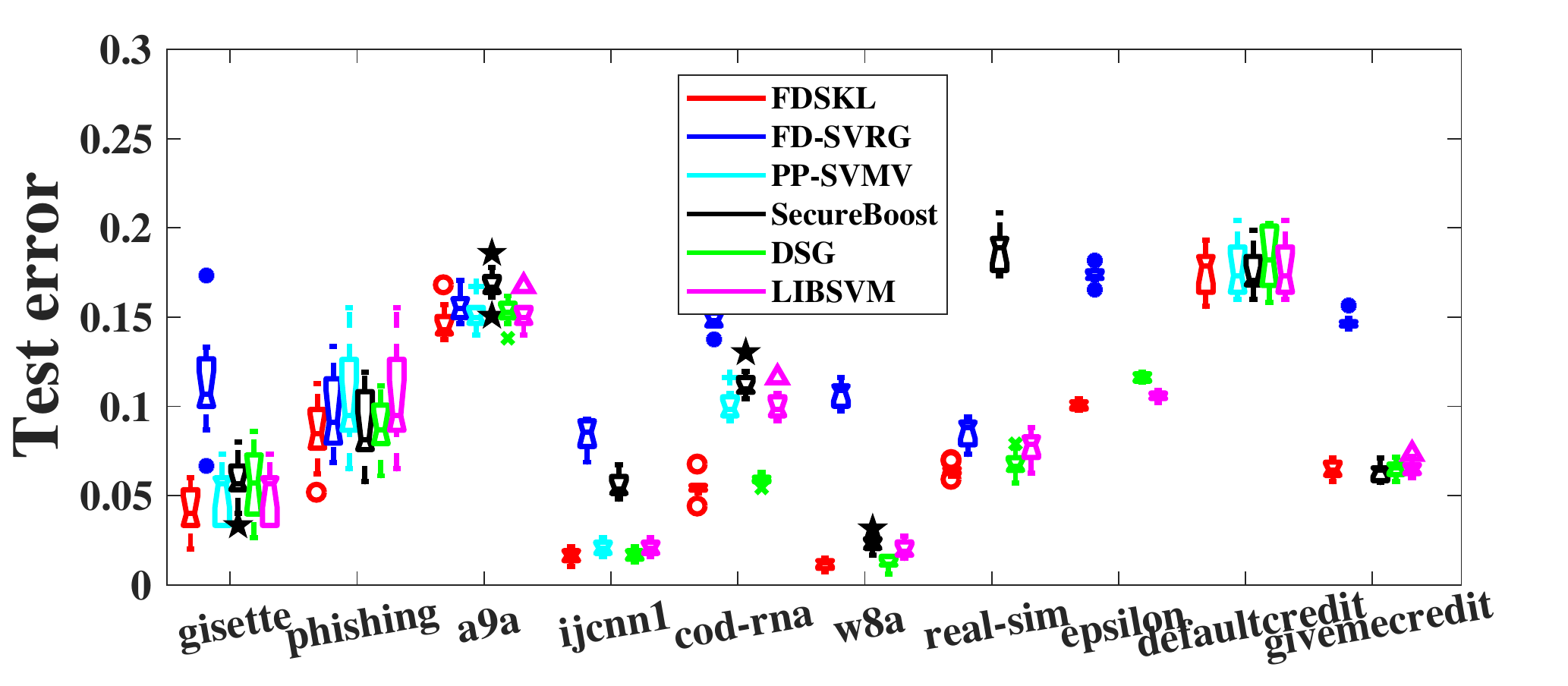}
	%\vspace*{-10pt}
	\caption{The boxplot of test errors of three state-of-the-art kernel methods, tree-boosting method (SecureBoost),  linear method (FD-SVRG) and our FDSKL.}
	\label{results_num}
	%\vspace*{-6pt}
\end{figure*}
	\begin{table}[htbp]
	\centering
%	\vspace*{-5pt}
%	\small
	\caption{The comparison  of total training time  between our  FDSKL  and SecureBoost algorithm.}
%	\vspace*{-8pt}
	\setlength{\tabcolsep}{3.5mm}
	\begin{tabular}{c|c|c}
		\hline
		\textbf{Datasets} &  \textbf{FDSKL  (min)} &  \textbf{SecureBoost (min)}  \\
		\hline
		gisette & 1 & 46\\
		phishing  & 1 & 10\\
		a9a & 2 & 15\\
		ijcnn1 & 1 & 17\\
		cod-rna  & 1 & 16\\
		w8a  & 1 & 17\\
		real-sim  & 14 & 65\\
		epsilon  & 29 & $>$900\\
	    \hline
	    defaultcredit & 1 & 13\\
	    givemecredit & 3 & 32\\
		\hline
	\end{tabular}
	\label{datasets_time_comparsion}
%	\vspace*{-3pt}
\end{table}
Since the official implementation of SecureBoost algorithm did not provide the elapsed time of each iteration, we present the total training time between our FDSKL and SecureBoost in Table \ref{datasets_time_comparsion}. From these results, we could make a statement that SecureBoost algorithm is not suitable for the high dimension or large scale datasets which limits its generalization.

Finally, we present the  test errors of three state-of-the-art
kernel methods, tree-boosting method (SecureBoost), linear method (FD-SVRG) and our FDSKL in Figure \ref{results_num}. All results are averaged over 10 different train-test split trials \citep{gu2015new}.  From the results, we find that our FDSKL always has the lowest test error and variance.  In addition, tree-boosting method SecureBoost performs poor in high-dimensional datasets such as \textit{real-sim}. And the linear method normally has worse results than other kernel methods.

%\vspace*{-3pt}
\section{Conclusion} \label{section_concluding}
Privacy-preservation federated learning for vertically partitioned data is urgent currently in data mining and machine learning.
In this paper, we proposed a federated doubly stochastic kernel learning (\emph{i.e.}, FDSKL) algorithm for vertically partitioned data, which breaks the limitation of implicitly linear
separability used in the existing privacy-preservation
federated learning algorithms.
We proved that FDSKL  has a sublinear convergence rate, and can guarantee
data security under the semi-honest assumption. To the best of our knowledge, FDSKL is the first efficient and  scalable privacy-preservation federated kernel method. Extensive experimental results show that FDSKL is more efficient than the existing state-of-the-art kernel methods for high dimensional data while  retaining the similar generalization performance.

\nocite{langley00}

%\section*{Acknowledgments}
%This work was supported
%by the Natural Sciences and Engineering Research Council of
%Canada (NSERC) and the NSF of China (61232016, 61573191 and 61202137).

\section*{Appendix A: Proof of Theorem  \ref{thm1}}  \label{firstthm0.5}
%For Assumption 2, since we assume bounded domain and generate bounded functions $f_t$, $M$ always exists and less than $\infty$.
We first prove that the output of Algorithm \ref{protocol1} is actually  $\omega_i^T x+b$ in Lemma \ref{lemma1}   which is the basis of  FDSKL.

\setcounter{theorem}{0}

\begin{lemma}\label{lemma1}
	The output  of Algorithm \ref{protocol1} (\emph{i.e.}, $ \sum_{\hat{\ell}=1}^q  \left ( (\omega_i)_{\mathcal{G}_{\hat{\ell}}}^T (x)_{\mathcal{G}_{\hat{\ell}}}  +b^{\hat{\ell}} \right ) - \overline{b}^{\ell'}$  is equal to $\omega_i^T x+b$,  where each $b^{\hat{\ell}}$ and $b$ are drawn from a uniform distribution on $[0,2\pi]$, $\overline{b}^{\ell'}=\sum_{{\hat{\ell}} \neq \ell'} b^{{\hat{\ell}}}$, and $\ell' \in \{ 1,\ldots,q\} -\{\ell \} $.
\end{lemma}
\begin{proof} The proof has two parts: 1) Because each worker knows the same  probability measure $\mathbb{P}(\omega)$, and  has a same  random  seed $i$, each worker can  generate the same random feature $\omega$ according to the same random  seed $i$.  Thus, we can locally compute  $(\omega_i)_{\mathcal{G}_\ell}^T (x)_{\mathcal{G}_\ell}$, and get $\sum_{\ell=1}^q   (\omega_i)_{\mathcal{G}_\ell}^T (x)_{\mathcal{G}_\ell} =\omega_i^T x$. 2)  Because each  $b^{\hat{\ell}}$ is drawn from a uniform distribution on $[0,2\pi]$, we can have that $\sum_{\hat{\ell}=1}^q b^{\hat{\ell}} - \overline{b}^{\ell'}=\sum_{\ell=1}^q b^{\hat{\ell}} - \sum_{\hat{\ell} \neq \ell'} b^{{\hat{\ell}}}$ is also drawn from a uniform distribution on $[0,2\pi]$. This completes the proof.
\end{proof}

Based on Lemma \ref{lemma1}, we can conclude that the federated learning algorithm (\emph{i.e.}, FDSKL)  can produce  the doubly stochastic gradients exactly  same to the ones of DSG algorithm with constant learning rate.  Thus, under Assumption 1, we  can prove that FDSKL converges to the optimal solution almost at a rate of $\mathcal{O}(1/t)$ as shown in Theorem \ref{thm1}.

Now we will show that Algorithm \ref{algorithm3} with constant stepsize is convergent in terms of $||f-f_*||^2_{\cal H}$ with a rate of near $O(1/t)$, where $f^*$ is the minimizer of problem. And then we decompose the error according to its sources, which include the error from random features $|f_t(x)-h_t(x)|^2$, and the error from data randomness $||h_t(x)-f_*(x)||^2_{\cal H}$, that is:

\begin{equation}\label{equ1}
|f_{t}(x) - f_{*}(x)|^2 \le 2|f_{t}(x) - h_{t}(x)|^2 + 2\kappa||h_{t} - f_{*}||_{\cal H}^2
\end{equation}

Before we show the main Theorem \ref{theorem1}, we first give several following lemmas. Note that, different from the proof of \citep{dai2014scalable}, we obtain the Lemma \ref{lemma7'} by the summation formula of geometric progression and the Lemma \ref{lemma9'} inspired by the proof of the recursive formula in \citep{recht2011hogwild}. Firstly, we give the convergence analysis of error due to random feature. Let $\mathcal{D}_t$ denote the subset of $\mathcal{S}$ which have been picked up at the $t$-th iteration of FDSKL.

\setcounter{theorem}{2}
\begin{lemma}
	\label{lemma7'}
	For any $x \in \cal X$ and $\frac{1}{\lambda}>\gamma > 0$,
	\begin{equation}
	\mathbb{E}_{{\cal D}_t,  \omega_t}\Big[\Big|f_{t+1}(x) - h_{t+1}(x)\Big|^2 \Big] \le C^2
	\end{equation}
	where $ C^2 := M^2(\sqrt{\kappa}+\sqrt{\phi})^2\gamma/\lambda$.
\end{lemma}
\begin{proof}
	We denote $W_i(x)=W_i(x;\mathcal{D}^i,\omega^i):=\alpha^t_i(\zeta_i(x)-\xi_i(x))$. According to the above assumptions, $W_i(x)$ have a bound:
	\begin{equation*}
	|W_i(x)|\le c_i=|\alpha^t_i|(|\zeta_i(x)|+|\xi_i(x)|)=M(\sqrt{\kappa}+\sqrt{\phi})|\alpha^t_i|
	\end{equation*}
	
	Obviously, $ |\alpha^t_t|\le\gamma $, and $ |\alpha^t_i|\le\gamma\prod^{t}_{j=i+1}(1-\gamma c)\le\gamma $ where $(1-\gamma \lambda)\le 1$. Considering the summation formula of geometric progression, $\sum_{i=1}^{t}|\alpha^t_i|\le\frac{1}{\lambda}$. Consequently, $ \sum_{i=1}^{t}|\alpha^t_i|^2\le \frac{\gamma}{\lambda} $.
	
	Then we have:
	$\mathbb{E}_{{\cal D}_t,  \omega_t}\Big[\Big|f_{t+1}(x) - h_{t+1}(x)\Big|^2 \Big]=\sum_{i=1}^{t}|W_i(x)|^2\le M^2(\sqrt{\kappa}+\sqrt{\phi})^2\sum_{i=1}^{t}|\alpha^t_i|^2\le M^2(\sqrt{\kappa}+\sqrt{\phi})^2\gamma/\lambda$.
	
	%	Then by Azuma's Inequality, for any $\epsilon>0$,
	%	\begin{equation*}
	%	\Pr\limits_{{\cal D}^+_t, {\cal D}^-_t, \omega_t}\left\lbrace |\sum_{i=1}^{t}W_i(x)|\ge\epsilon\right\rbrace \le2\exp\left\lbrace\frac{2\epsilon^2}{\sum_{i=1}^{t}c_i^2}\right\rbrace
	%	\end{equation*}
	%	which is equivalent as:
	%	\begin{equation*}
	%	\Pr\limits_{{\cal D}^+_t, {\cal D}^-_t, \omega_t}\left\lbrace \left( \sum_{i=1}^{t}W_i(x)\right) ^2\ge\ln(2/\delta)\sum_{i=1}^tc_i^2/2\right\rbrace \le\delta
	%	\end{equation*}
	%	Moreover,
	%	\begin{equation*}
	%	\mathbb{E}_{{\cal D}^+_t, {\cal D}^-_t, \omega_t}\left[\left( \sum_{i=1}^{t}W_i(x)\right) ^2 \right] =\int_0^\infty 	\Pr\limits_{{\cal D}^+_t, {\cal D}^-_t, \omega_t}\left\lbrace\left( \sum_{i=1}^{t}W_i(x)\right)^2\ge\epsilon\right\rbrace d\epsilon=\int_0^\infty 2\exp\left\lbrace -\frac{2\epsilon}{\sum_{i=1}^tc_i^2} \right\rbrace d\epsilon=\sum_{i=1}^tc_i^2
	%	\end{equation*}	
	%	Cause $f_{t+1}(x)-h_{t+1}(x)=\sum_{i=1}^tW_i(x)$,
	we obtain the above lemma.
\end{proof}

%\setcounter{theorem}{0}
%	\begin{remark}
%		%\textbf{Remarks}.
%		Intuitively speaking, with properly chosen learning rates, the error introduced by stochastically sampled random features will shrink at the rate of $O(1/t)$. For implementation, $\theta c \in (1,2)$ can be used as a guiding rule in choosing the learning rate.
%	\end{remark}
The next lemma gives the convergence analysis of error due to random data, whose derivation actually depends on the results of the previous lemmas.
\begin{lemma}
	\label{lemma9'}
	Set $\frac{1}{\lambda}>\gamma > 0$, with $\gamma=\frac{\epsilon\vartheta}{2B}$ for $\vartheta\in\left( \right.0,1\left.\right]$, we will reach
	$
	\mathbb{E} \Big[||h_{t} - f_{*}||_{\cal H}^2\Big] \le \epsilon
	$ after \begin{equation}\label{tsolu}
	t \geq \frac{2B\log (2e_1/\epsilon)}{\vartheta\epsilon \lambda}
	\end{equation} iterations,
	where $ B=\left[\sqrt{G_2^2+G_1}+G_2\right]^2$, $G_1=\frac{2\kappa M^2}{\lambda}$, $G_2=\frac{\kappa^{1/2}M(\sqrt{\kappa}+\sqrt{\phi})}{2\lambda^{3/2}}$ and $e_1=\mathbb{E}[\|h_1-f_*\|_{\mathcal{H}}^2]$.
\end{lemma}
\begin{proof}
	In order to simplify the notations, let us denote the following three different gradient terms, they are:
	\begin{align*}
	&g_t=\xi_t+ch_t=L_{t}'(f_t)k(x_t,\cdot)+\lambda h_t \\
	&\hat{g}_t=\hat{\xi_t}+ \lambda h_t=L_{t}'(h_t)k(x_t,\cdot)+\lambda h_t\\
	&\bar{g}_t=\mathbb{E}_{{\cal D}_t}[\hat{g}_t]=\mathbb{E}_{{\cal D}_t}[L_{t}'(h_t)k(x_t,\cdot)]+\lambda h_t
	\end{align*}
	From previous definition, we have $h_{t+1}=h_t-\gamma g_t,\forall t\ge1$.
	
	We denote $A_t=\|h_t-f_*\|_{\mathcal{H}}^2$, then we have
	\begin{align*}
	A_{t+1}&=\|h_t-f_*-\gamma g_t\|_{\mathcal{H}}^2\\&=A_t+\gamma^2\|g_t\|_{\mathcal{H}}^2-2\gamma\langle h_t-f_*,g_t\rangle_{\mathcal{H}}\\
	&=A_t+\gamma^2\|g_t\|-2\gamma\langle h_t-f_*,\bar{g}_t\rangle_{\mathcal{H}}+2\gamma\langle h_t-f_*,\bar{g}_t-\hat{g}_t\rangle_{\mathcal{H}}+2\gamma\langle h_t-f_*,\hat{g}_t-g_t\rangle_{\mathcal{H}}
	\end{align*}
	Cause the strongly convexity of objective and optimality condition, we have
	\begin{equation*}
	\langle h_t-f_*,\bar{g}_t\rangle_{\mathcal{H}}\ge \lambda \|h_t-f_*\|_{\mathcal{H}}^2
	\end{equation*}
	Hence, we have
	\begin{equation}
	A_{t+1}\le (1-2\gamma \lambda )A_t+\gamma^2\|g_t\|_{\mathcal{H}}^2+2\gamma\langle h_t-f_*,\bar{g}_t-\hat{g}_t\rangle_{\mathcal{H}}+2\gamma\langle h_t-f_*,\hat{g}_t-g_t\rangle_{\mathcal{H}},\forall t\ge 1
	\end{equation}
	
	Let us denote $\mathcal{M}_t=\|g_t\|_{\mathcal{H}}^2,\mathcal{N}_t=\langle h_t-f_*,\bar{g}_t-\hat{g}_t\rangle_{\mathcal{H}},\mathcal{R}_t=\langle h_t-f_*,\hat{g}_t-g_t\rangle_{\mathcal{H}}$. According to Lemma 3, $\mathcal{M}_t,\mathcal{N}_t,\mathcal{R}_t$ can be bounded. Then we denote $e_t=\mathbb{E}_{{\cal D}_t, \omega_t}[A_t]$, given the above bounds, we obtain the following recursion,
	\begin{equation}
	\label{et+1}
	e_{t+1}\le (1-2\gamma \lambda )e_t+4\kappa M^2
	\gamma^2+2\kappa^{1/2}L\gamma C\sqrt{e_t}	
	\end{equation}
	
	When $\gamma>0$ and $|a_t^i|\le\gamma,\forall 1\le i\le t$. Consequently,  $ C^2 \le M^2(\sqrt{\kappa}+\sqrt{\phi})^2\gamma/\lambda$. Applying these bounds to the above recursion, we have
	\begin{equation}
	\label{simet+1}
	e_{t+1}\le \beta_1e_t+\beta_2+\beta_3\sqrt{e_t},
	\end{equation}
	
	Note that $\beta_1=1-2\gamma \lambda$, $\beta_2=4\kappa M^2
	\gamma^2$ and $\beta_3=\frac{ 2\kappa^{1/2}\gamma^{3/2} LM(\sqrt{\kappa}+\sqrt{\phi})}{\sqrt{\lambda}}$. Then consider that when $t\rightarrow\infty$  we have:
	\begin{equation}
	\label{einfty}
	e_{\infty}=\beta_1e_{\infty}+\beta_2+\beta_3\sqrt{e_{\infty}}
	\end{equation}
	and the solution of the above recursion (\ref{einfty}) is \begin{equation}
	\label{einftysolu}
	e_{\infty}=\left[\sqrt{\frac{\beta_3^2}{16\gamma^2 \lambda^2}+\frac{\beta_2}{2\gamma \lambda}}+\frac{\beta_3}{4\gamma \lambda}\right]^2=\gamma*\left[\sqrt{G_2^2+G_1}+G_2\right]^2
	\end{equation}
	where $G_1=\frac{2\kappa M^2}{\lambda}$ and $G_2=\frac{\kappa^{1/2}M(\sqrt{\kappa}+\sqrt{\phi})}{2 \lambda^{3/2}}$.
	
	We use Eq. (\ref{simet+1}) minus Eq. (\ref{einfty}), then we get:
	\begin{align}
	e_{t+1}&\le \beta_1(e_t-e_{\infty})+\beta_3(\sqrt{e_t}-\sqrt{e_{\infty}})+e_{\infty}\nonumber\\
	&\le \beta_1(e_t-e_{\infty})+\beta_3(\frac{e_t-e_{\infty}}{2\sqrt{e_{\infty}}})+e_{\infty}\nonumber\\
	&\le (\beta_1+\frac{\beta_3}{2\sqrt{e_{\infty}}})(e_t-e_{\infty})+e_{\infty}\nonumber\\
	&\le (1-\gamma \lambda)(e_t-e_{\infty})+e_{\infty}\label{efinal}
	\end{align}
	where the second inequality is due to $a-b\leq\frac{a^2-b^2}{2b} $, and the last step is due to $\frac{\beta_3}{2\sqrt{e_{\infty}}}=\frac{1}{2\left[\sqrt{\frac{1}{16\gamma^2 \lambda^2}+\frac{\beta_2}{2\gamma \lambda\beta_3^2}}+\frac{1}{4\gamma \lambda}\right]}\le\frac{1}{2[\frac{1}{4\gamma \lambda}+\frac{1}{4\gamma \lambda}]}=\gamma \lambda$.
	
	We can easily apply Eq. (5.1) in \citep{recht2011hogwild} to Eq. (\ref{efinal}), and Eq. (\ref{einftysolu}) satisfy $a_{\infty}(\gamma)\leq\gamma B$. We denote $B=\left[\sqrt{G_2^2+G_1}+G_2\right]^2$. Particularly, unwrapping (\ref{efinal}) we have:
	\begin{equation}\label{einftysolusum}
	e_{t+1}\leq(1-\gamma \lambda)^t(e_1-e_{\infty})+e_{\infty}
	\end{equation}
	
	Suppose we want this quantity (\ref{einftysolusum}) to be less than $\epsilon$. Similarly, we let both terms are less than $\epsilon/2$, then for the second term, we have \begin{equation}\label{second}
	\gamma\leq\frac{\epsilon}{2B}=\frac{\epsilon}{2\left[\sqrt{G_2^2+G_1}+G_2\right]^2}
	\end{equation}
 For the first term, we need:
	\begin{equation*}
	(1-\gamma \lambda)^te_{1}\leq \frac{\epsilon}{2}
	\end{equation*}
	which holds if \begin{equation}\label{first}
	t\geq \frac{\log (2e_1/\epsilon)}{\gamma \lambda}
	\end{equation}
	
	According to the (\ref{second}), we should pick $\gamma=\frac{\epsilon\vartheta}{2B}$ for $\vartheta\in\left( \right.0,1\left.\right]$. Combining this with Eq. (\ref{first}), after
	\begin{equation}\label{tsolu}
	 t\geq \frac{2B\log (2e_1/\epsilon)}{\vartheta\epsilon \lambda}
	\end{equation} iterations we will have $e_{t}\leq\epsilon$ and that give us a $\mathcal{O}(1/t)$ convergence rate, if eliminating the $\log(1/\epsilon)$ factor.
	
\end{proof}
Now we give the technical lemma which is used in proving Lemma \ref{lemma10'}.
\begin{lemma}
	\label{lemma10'}
	%	This is a technical lemma which is essential for proving Lemma \ref{lemma9'}):
	\begin{equation}\label{lemma10_1}
	\mathcal{M}_t \le 4\kappa M^2;
	\end{equation}
	\begin{equation}\label{lemma10_2}
	\mathbb{E}_{{\cal D}_t,  \omega_t}\Big[{\cal N}_t\Big] = 0
	\end{equation}
	\begin{equation}\label{lemma10_3}
	\mathbb{E}_{{\cal D}_t,  \omega_t}\Big[{\cal R}_t\Big] \le \kappa^{1/2}\mathcal{L} C\sqrt{\mathbb{E}_{{\cal D}_{t-1},  \omega_{t-1}} [A_t]}
	\end{equation}
	where
	$A_t = ||h_{t} - f_{*}||^2_{\cal H}$.
\end{lemma}
\begin{proof}
	Firstly, let we prove the Lemma \ref{lemma10'}.1 (\ref{lemma10_1}):
	\begin{equation*}
	\mathcal{M}_t=\|g_t\|_{\mathcal{H}}^2=\|\xi_t+\lambda h_t\|_{\mathcal{H}}^2\le\left(\|\xi_t\|_{\mathcal{H}}+\lambda \|h_t\|_{\mathcal{H}}\right)^2
	\end{equation*}
	and
	\begin{equation*}
	\|\xi_t\|_{\mathcal{H}}=\|L'(f_t(x_t),y_t)k(x_t,\cdot)\|_{\mathcal{H}}\le\kappa^{1/2}M
	\end{equation*}	
	Then we have:
	\begin{align*}
	\|h_t\|_{\mathcal{H}}^2
	&=\sum_{i=1}^{t-1}\sum_{j=1}^{t-1}\alpha^{t-1}_i \alpha^{t-1}_j L'(f_i(x_i),y_i)L'(f_j(x_j),y_j)k(x_i,x_j)\nonumber\\
	&\le \kappa M^2\sum_{i=1}^{t-1}\sum_{j=1}^{t-1}|\alpha^{t-1}_i||\alpha^{t-1}_j|
	\end{align*}
	Consequently, $\|h_t\|_{\mathcal{H}}\leq\kappa^{1/2}M\sqrt{\sum_{i,j=1}^{t-1}|\alpha^{t-1}_i||\alpha^{t-1}_j|}\leq\kappa^{1/2}M\frac{1}{\lambda}$. Then the above equation can be rewritten as:
	\begin{equation*}
	\mathcal{M}_t=\|g_t\|_{\mathcal{H}}^2=\|\xi_t+ \lambda h_t\|_{\mathcal{H}}^2\le\left(\|\xi_t\|_{\mathcal{H}}+\lambda \|h_t\|_{\mathcal{H}}\right)^2\le(\kappa^{1/2}M+\lambda \times \kappa^{1/2}M\frac{1}{\lambda})^2=4\kappa M^2
5	\end{equation*}
	Therefore, we finish the proof of Lemma \ref{lemma10'}.1.
	
	For the second one (\ref{lemma10_2}), ${\cal N}_t = \langle h_t-f_*,\bar{g}_t-\hat{g}_t\rangle_{\mathcal{H}}$, then we have:
	\begin{align*}
	\mathbb{E}_{{\cal D}_t,  \omega_t}\Big[{\cal N}_t\Big]
	&=\mathbb{E}_{{\cal D}_{t-1},  \omega_t}\left[ \mathbb{E}_{{\cal D}_{t-1}}[\langle h_t-f_*,\bar{g}_t-\hat{g}_t\rangle_{\mathcal{H}}|{\cal D}_{t-1},  \omega_t]\right] \nonumber\\
	&=\mathbb{E}_{{\cal D}_{t-1}, \omega_t}\left[ \langle h_t-f_*,\mathbb{E}_{{\cal D}_{t}}[\bar{g}_t-\hat{g}_t]\rangle_{\mathcal{H}}\right] \nonumber\\
	&=0 \nonumber
	\end{align*}
	
	The third one (\ref{lemma10_3}), ${\cal R}_t=\langle h_t-f_*,\hat{g}_t-g_t\rangle_{\mathcal{H}}$, then we have:
	\begin{align*}
	\mathbb{E}_{{\cal D}_t, \omega_t}\Big[{\cal R}_t\Big]
	&=\mathbb{E}_{{\cal D}_t,  \omega_t}\Big[\langle h_t-f_*,\hat{g}_t-g_t\rangle_{\mathcal{H}}\Big]\\
	&=\mathbb{E}_{{\cal D}_t,  \omega_t}\Big[\langle h_t-f_*,[l'(f_t(x_t),y_t)-l'(h_t(x_t),y_t)]k(x_t,\cdot)\rangle_{\mathcal{H}}\Big]\\
	&\le \mathbb{E}_{{\cal D}_t, \omega_t}\Big[\| h_t-f_*\|_{\mathcal{H}}\cdot|l'(f_t(x_t),y_t)-l'(h_t(x_t),y_t)|\cdot\|k(x_t,\cdot)\|_{\mathcal{H}}\Big]\\
	&\le \kappa^{1/2}\mathcal{L}\cdot\mathbb{E}_{{\cal D}_t,  \omega_t}\Big[\| h_t-f_*\|_{\mathcal{H}}\cdot |f_t(x_t)-h_t(x_t)|\Big]\\
	&\le \kappa^{1/2}\mathcal{L}\cdot\sqrt{\mathbb{E}_{{\cal D}_t,  \omega_t}\| h_t-f_*\|_{\mathcal{H}}^2}\sqrt{ \mathbb{E}_{{\cal D}_t,  \omega_t}|f_t(x_t)-h_t(x_t)|^2}\\
	&\le \kappa^{1/2}\mathcal{L}\cdot C\sqrt{\mathbb{E}_{{\cal D}_{t-1}, \omega_{t-1}} [A_t]}
	\end{align*}
	where the first and third inequalities are due to Cauchy-Schwarz Inequality and the second inequality is due to the Assumptions 1. And the last step is due to the Lemma 1 and the definition of $A_t$.
\end{proof}

\setcounter{theorem}{1}
According to Lemmas \ref{lemma7'} and \ref{lemma9'}, we obtain the final results on convergence in expectation:
\begin{theorem} [Convergence in expectation]
	\label{theorem1}
	Set  $\epsilon>0$,  $\min \{ \frac{1}{\lambda}, \frac{\epsilon \lambda}{4M^2(\sqrt{\kappa}+\sqrt{\phi})^2}\}>\gamma > 0$, for Algorithm \ref{algorithm3}, with $\gamma=\frac{\epsilon\vartheta}{8\kappa B}$ for $\vartheta\in\left( \right.0,1\left.\right]$, we will reach
	$
	\mathbb{E} \Big[|f_{t }(x) - f_*(x)|^2\Big] \le \epsilon
	$ after \begin{equation}\label{tsolu}
	t \geq \frac{8\kappa B\log (8\kappa e_1/\epsilon)}{\vartheta\epsilon \lambda}
	\end{equation} iterations, where $B\text{ and }e_1$ are as defined in Lemma \ref{lemma9'}.
	
\end{theorem}
\begin{proof}
	Substitute Lemma \ref{lemma7'} and \ref{lemma9'} into the inequality (\ref{equ1}), i.e.
	\begin{align*}
	&\mathbb{E}\Big[|f_{t}(x) - f_{*}(x)|^2\Big]\\
	&\le 2\mathbb{E} \Big[|f_{t}(x) - h_{t+1}(x)|^2\Big] + 2\kappa\mathbb{E} \Big[||h_{t} - f_{*}||^2_{\cal H}\Big]
\\
	&\le \epsilon
	\end{align*}
	
	Again, it is sufficient that both terms are less than $\epsilon/2$. For the second term, we can directly derive from Lemma \ref{lemma9'}. As for the first term, we can get a upper bound of $\gamma$: $\frac{\epsilon \lambda}{4M^2(\sqrt{\kappa}+\sqrt{\phi})^2}$. Then we obtain the above theorem.
	
\end{proof}
\section*{Appendix B: Proof of Lemma \ref{lemma2}}
\setcounter{theorem}{1}
\begin{lemma}\label{lemma2}
Using a  tree structure $T_2$ for workers  $\{1, \ldots,q\}-\{\ell'  \}$  which is totally different to  the tree $T_1$ to compute $\overline{b}^{\ell'}=\sum_{{\hat{\ell}} \neq \ell'} b^{{\hat{\ell}}}$, for any worker, there is no risk to  disclose the value of $b^{{{\hat{\ell}}}}$ or the sum of $b^{{{\hat{\ell}}}}$ on other workers.
\end{lemma}
\begin{proof}
If there exist the inference attack, the inference attack must happen in one non-leaf node in the tree $T_1$. We denote the non-leaf node as NODE.
 %We consider two cases for the non-leaf nodes: one is the  node corresponding a subtree with three nodes (denoted as Node-1), and the other one is the node corresponding a subtree with more than three nodes (denoted as Node-2).

 Assume the $\ell$-th worker is one of the leafs of  the subtree of NODE. If the $\ell$-th worker wants to start the inference attack, it  is  necessary to have the sum of all $b^{\ell'}$ on the leaves of the tree, which means that the subtree corresponding to NODE also belongs to $T_2$.
\end{proof}
\section*{Appendix C: Complexity Analysis of FDSKL}
We derive the computational complexity and communication cost of FDSKL  as follows.
\begin{enumerate}[leftmargin=0.2in]
\item The line \ref{step2} of Algorithm \ref{algorithm3} picks up an instance $(x_i)_{\mathcal{G}_\ell}$ from the local data $D^\ell$ with index $i$. Thus, its computational complexity of all workers is $\mathcal{O}(q)$, and there is no communication  cost.
\item The line \ref{step3} of Algorithm \ref{algorithm3}  sends $i$ to other workers  using a reverse-order tree structure. Thus, its computational complexity is $\mathcal{O}(1)$, and the communication  cost is $\mathcal{O}(q)$.
\item  The line \ref{step4} of Algorithm \ref{algorithm3} samples $\omega_i \sim \mathbb{P}(\omega)$ with the random  seed $i$. Thus, its computational complexity of all workers is $\mathcal{O}(dq)$, and there is no communication  cost.
\item The line \ref{step5} of Algorithm \ref{algorithm3} computes $\omega_i^T x_i+b $. The detailed procedure is presented in Algorithm \ref{protocol1}. Its computational complexity is $\mathcal{O}(d+q)$, and the communication  cost is $\mathcal{O}(q)$. The detailed analysis of computational complexity and communication cost of Algorithm \ref{protocol1} is omitted here.
\item  The lines \ref{step6}-\ref{step7} of Algorithm \ref{algorithm3} uses the tree-structured communication scheme to compute $f(x_i)=\sum_{\ell=1}^q f^\ell(x_i)$. Because the computational complexity and communication  cost of $f^\ell(x)$ are $\mathcal{O}(dq |\Lambda^\ell|)$ and $\mathcal{O}(q |\Lambda^\ell|)$, respectively, as analyzed in the next part, we can conclude that the computational complexity of lines \ref{step6}-\ref{step7} of Algorithm \ref{algorithm3} is $\mathcal{O}(dq t)$, and its communication  cost is $\mathcal{O}(q t)$, where $t$ is the global iteration number.
\item The line \ref{step8}-\ref{step9} of Algorithm \ref{algorithm3} compute  the current coefficient  $\alpha_i$. Thus, its computational complexity  is $\mathcal{O}(1)$, and there is no communication  cost.
\item The line \ref{step10} of Algorithm \ref{algorithm3} updates the former coefficients. Thus, its computational complexity of all workers is $\mathcal{O}(t)$,  and its communication cost is $\mathcal{O}(q)$.
\end{enumerate}
Based on the above discussion, the computational complexity for one iteration of FDSKL is $\mathcal{O}(dqt)$. Thus, the total computational complexity of FDSKL is $\mathcal{O}(dqt^2)$. Further, the    communication cost for one iteration of FDSKL is $\mathcal{O}(qt)$, and the total communication cost  of FDSKL is $\mathcal{O}(qt^2)$.

We  derive the computational complexity and communication  cost of Algorithm \ref{alg:predict} as follows.
\begin{enumerate}[leftmargin=0.2in]
\item The line \ref{step1} of Algorithm \ref{alg:predict} sets the initial solution to $f^\ell(x)$.  Thus, its computational complexity  is $\mathcal{O}(1)$, and there is no communication  cost.
\item The lines \ref{step3}-\ref{step5} of Algorithm \ref{alg:predict} compute $\phi_{\omega_i}(x)$ which are  similar to the lines \ref{step4}-\ref{step5} of   Algorithm \ref{algorithm3}. According to the previous analyses in this section, we have that  its computational complexity  is $\mathcal{O}(dq)$, and the communication  cost is $\mathcal{O}(q)$.
\item The line \ref{step6} of Algorithm \ref{alg:predict} is updating the value of $f^\ell(x)$. Thus, its computational complexity  is $\mathcal{O}(1)$, and there is no communication  cost.
\end{enumerate}
Thus, the  computational complexity of Algorithm \ref{alg:predict} is $\mathcal{O}(dq |\Lambda^\ell|)$ and the communication  cost of Algorithm \ref{alg:predict} is $\mathcal{O}(q |\Lambda^\ell|)$.

\bibliography{example_paper}

\end{document}